\documentclass[12pt,draftcls,onecolumn]{IEEEtran}
\usepackage{amsmath,amssymb,verbatim}
 \usepackage[usenames,dvipsnames]{pstricks}
  \usepackage{epsfig}
\usepackage{graphicx}
\usepackage{bbm}
\usepackage{url}
\usepackage{float,psfrag,epsfig}
\usepackage{wrapfig}
\usepackage{algorithm}
\usepackage{algorithmic}
\usepackage{rotating}
\usepackage{multirow}
\newcommand{\BlackBox}{\rule{1.5ex}{1.5ex}}  
\newenvironment{proof}{\par\noindent{\bf Proof\ }}{\hfill\BlackBox\\[2mm]}
 
\newtheorem{theorem}{Theorem}
\newtheorem{lemma}[theorem]{Lemma}

\DeclareMathAlphabet{\mathpzc}{OT1}{pzc}{m}{it}
\newcommand{\prob}{\Pr}
\newcommand{\expect}{\mathbbm{E}}
\newcommand{\history}{\mathcal{H}}
\newcommand{\myset}[1]{\mathcal{#1}}
\newcommand{\Oo}{\myset{O}}

\newcommand{\Tt}{\myset{T}}
\newcommand{\Nn}{\myset{N}}
\newcommand{\Ss}{\myset{S}}
\newcommand{\Ff}{\myset{F}}

\newcommand{\E}{\mathbb{E}}

\newcommand{\supp}{\mathsf{supp}}
\newtheorem{property}{Property}
\newcommand{\lst}[1]{\ensuremath{\prec_{#1}}}
\newcommand{\lsteq}[1]{\ensuremath{\preccurlyeq_{#1}}}
\newcommand{\grt}[1]{\ensuremath{\succ_{#1}}}

\newcommand{\eqvl}[1]{\ensuremath{\sim}_{#1}}
\newcommand{\oracle}{\ensuremath{\text{Oracle}}}

\newcommand{\order}{\texttt{order}}
\newcommand{\add}{\texttt{add}}
\DeclareFixedFont{\auacc}{OT1}{phv}{m}{n}{12}

\title{From Small-World Networks to Comparison-Based Search}

\author{Amin Karbasi,  \IEEEmembership{Member, IEEE},
Stratis Ioannidis, \IEEEmembership{Member, IEEE},
and Laurent Massouli\'e,  \IEEEmembership{Member, IEEE}
\thanks{{\footnotesize 
     Amin Karbasi is with the department of computer science at ETHZ, Switzerland (email: amin.karbasi@inf.ethz.ch). 
      Stratis Ioannidis is with the Technicolor Palo Alto, CA-94301, USA (email: stratis.ioannidis@technicolor.com). Laurent Massouli is with the MSR-INRIA joint research center, \'Ecole Polytechnique, Paris, France (email: laurent.massoulie@inria.fr).
}}  }

\begin{document}

\maketitle

\begin{abstract}
The problem of content search through comparisons has recently received considerable attention. In short, a user searching for a target
 object navigates through a database in the following manner: the user is asked to
 select the object most similar to her target from a small list of objects. A new object list
 is then presented to the user based on her earlier selection. This process is repeated
until the target is included in the list presented, at which point the search terminates.

This problem is known to be strongly related to the small-world network design problem. However, contrary to prior work, which focuses on cases where objects in the database are equally popular, we consider here the case where the demand for objects may be heterogeneous. 


We show that, under heterogeneous demand, the small-world network design problem is NP-hard. 
Given the above negative result, we propose a novel mechanism  for small-world design and provide an upper bound on its performance under heterogeneous demand. The above mechanism has a natural equivalent in  the context of  content search through comparisons, and we establish both an upper bound and a lower bound for the performance of this mechanism. These bounds are intuitively appealing, as they depend on the entropy of the demand  as well as its doubling constant, a quantity capturing the topology of the set of target objects. They also illustrate interesting connections between comparison-based search to classic results from information theory. Finally, we propose  an adaptive learning algorithm for content search that meets the performance guarantees achieved by the above mechanisms.

\end{abstract}
\pagenumbering{arabic}

\section{Introduction}\label{section:intro}

 \IEEEPARstart{T}{h}e problem we study in this paper is content search through comparisons. 
In short, a user searching for a target object navigates through a database in the following manner. The user is asked to select the object most similar to her target from a small list of objects. 
 A new object list is then presented to the user based on her earlier selection.  This  process is repeated until the target is included in the list presented, at which point the search terminates. 

Searching through comparisons is typical example of  \emph{exploratory search}~\cite{exploratory}, the need for which arises when users are unable to state and submit explicit queries to the database. Exploratory search has several important real-life applications. An often-cited example is navigating through a database of pictures of humans in which  subjects are photographed under diverse uncontrolled conditions~\cite{Facebrowsing,Suhas}. For example, the pictures may be taken outdoors, from different angles or distances, while the subjects assume different poses, are partially obscured, \emph{etc.} Automated methods may fail to extract meaningful features from such photos, so the database cannot be queried in the traditional fashion. On the other hand, a human searching for a particular person can easily select from a list of pictures the subject most similar to the person she has in mind. 

Users may also be unable to state queries because, \emph{e.g.}, they are unfamiliar with the search domain, or do not have a clear target in mind. For example, a novice classical music listener may not be able to express that she is, \emph{e.g.}, looking for a fugue or a sonata.   She might however identify among samples of different musical pieces the closest to the one she has in mind. 
Alternatively, a user surfing the web may not know a priori which post she wishes to read; presenting a list of blog posts and letting the surfer identify  which one she likes best can steer her in the right direction. 

In all the above applications, the problem of content search through comparisons amounts to determining which objects to present to the user in order to find the target object as quickly as possible. Formally, the behavior of a human user can be modeled by a so-called \emph{comparison oracle} introduced by \cite{Lifshits08}: given a target and a choice between two objects, the  oracle outputs the one closest to the target. The goal is thus to find a sequence of proposed pairs of objects that leads to the target object with as few oracle queries as possible. 
This problem was introduced by \cite{Lifshits08} and has recently received considerable attention (see, for example, \cite{Lifshits09,Suhas,Facebrowsing}). 

Content search through comparisons is also naturally related to the following problem: given a graph embedded in a metric space, how should one augment this graph by adding edges in order to minimize the expected cost of greedy forwarding over this graph? This is known as the  \emph{small-world network design} problem (see, for example, \cite{FraigneaudDoubling,Fraign}) and has a variety of applications as, \emph{e.g.}, in network routing. 
In this paper, we consider both  problems under the scenario of \emph{heterogeneous demand}. This is very interesting in practice: objects in a database are indeed unlikely to  be requested with the same frequency. 
Our contributions are as follows:

\begin{itemize}
\item  We show that the  small-world network design problem under general heterogeneous demand is NP-hard. Given earlier work on this problem under homogeneous demand \cite{Fraign,FraigneaudDoubling}, this result is interesting in its own right.
\item We propose a novel mechanism  for edge addition in the small-world design problem, and provide an upper bound on its performance. 
\item The above mechanism has a natural equivalent in  the context of content search through comparisons, and  we provide a matching upper bound for the performance of this mechanism.
\item  We also establish a lower bound on any mechanism solving the content search through comparisons problem. 
\item Finally, based on these results, we propose  an adaptive learning algorithm for content search that, given access only to a comparison oracle, can meet the performance guarantees achieved by the above mechanisms.
\end{itemize}

To the best of our knowledge, we are the first to study the above two problems in a setting of heterogeneous demand.  
Our analysis is intuitively appealing because our upper and lower bounds relate the cost of content search to two important properties of the demand distribution, namely its \emph{entropy} and its \emph{doubling constant}.
We thus provide performance guarantees in terms of the \emph{bias} of the distribution of targets, captured by the entropy, as well as the \emph{topology} of their embedding, captured by the doubling constant. 

The remainder of this paper is organized as follows. In Section~\ref{section:related} we provide an overview of the related work in this area. In Sections~\ref{section:definitions} and \ref{section:problemstatement} we introduce our notation and formally state the two problems that are the focus of this work, namely content search through comparisons and small-world network design.  We present our main results in Section~\ref{section:mainresults} and our adaptive learning algorithm in Section \ref{section:learning}. Section~\ref{section:proofs} is devoted to the proofs of our main theorems. We then address the two extensions of our work in Section~\ref{sec:extensions} and finally conclude in Section~\ref{section:conclusions}.

\section{Related Work}\label{section:related}
Content search through comparisons is a special case of nearest neighbour search (NNS), a problem that has been extensively studied \cite{Clarkson,Indyk98}. 
Our work can be seen as an extension of earlier work \cite{Karger,Krauthgamer,Clarkson} considering the NNS problem for objects embedded in a metric space with a small intrinsic dimension.
In particular, authors in~\cite{Krauthgamer} introduce navigating nets, a deterministic data structure for supporting  NNS  in  doubling metric spaces.  A similar technique was considered  by~\cite{Clarkson} 
for objects embedded in a space satisfying a certain sphere-packing property, while \cite{Karger} relied on growth restricted metrics; all of the above assumptions have connections to the doubling constant we consider in this paper.
In all of these works, however, the underlying metric space is fully observable by the search mechanism while, in our work, we are restricted  to accesses to a comparison oracle. Moreover, in all of the above works, the demand over the target objects is assumed to be homogeneous.

NNS with access to a comparison oracle   was first introduced by ~\cite{Lifshits08}, and further explored by~\cite{Lifshits09,Suhas,Facebrowsing}. A considerable advantage of the above works is that the assumption that objects are a-priori embedded in a metric space is removed; rather than requiring that similarity between objects is captured by a distance metric, the above works only assume  that any two objects can be ranked in terms of their similarity to any targer by the comparison oracle. 
To provide performance guarantees on the search cost, ~\cite{Lifshits08} introduced a so-called ``disorder-constant'', capturing the degree to which object rankings violate the triangle inequality. This disorder-constant plays roughly the same role in their analysis as the doubling constant does in ours. Nevertheless, these works also assume homogeneous demand, so our work can be seen as an extension of searching with comparisons to heterogeneity, with the caveat of restricting our analysis to the case where a metric embedding exists.

An additional important distinction between \cite{Lifshits08,Lifshits09,Suhas,Facebrowsing} and our work is the existence of a learning phase, during which
explicit questions are placed to the comparison oracle. A data-structure is constructed during this phase, which is subsequently used to answer queries submitted to the database during a ``search'' phase. The above works establish different tradeoffs between the length of the learning phase, the space complexity of the data structure created, and the cost incurred during searching.
 In contrast, the learning scheme we consider in Section~\ref{section:learning} is adaptive, and learning occurs while users search; the drawback lies in that our guarantees on the search cost are asymptotic. Again, the main advantage of our approach lies in dealing with heterogeneity.

The use of interactive methods (\emph{i.e.}, that incorporate human feedback) for content search has a long history in literature. Arguably, the first oracle considered to model such methods is the so-called membership oracle \cite{garey72},  which allows the search mechanism to ask a user questions of the form ``does the target belong to set $A$'' (see also our discussion in Section~\ref{section:entropy}). \cite{branson10}  deploys such an interactive method for object classification and evaluate it on the \textit{Animals with attributes} database. A similar approach was used by \cite{geman93} who formulated  shape recognition as a coding problem and  applied this approach to handwritten numerals and satellite images. Having access to a membership oracle however is a strong assumption, as humans may not necessarily be able to answer queries of the above type for \emph{any} object set $A$. Moreover, the large number of possible sets makes the cost of designing optimal querying strategies over large datasets prohibitive. In contrast, the comparison oracle model makes a far weaker assumption on human behavior---namely, the ability to compare different objects to the target---and significantly limits the design space, making search mechanisms using comparisons practical even over large datasets. 

The design of small-world networks (also called navigable networks) has received a lot of attention after the seminal work of \cite{kleinberg}. 
Our work is most similar to~\cite{FraigneaudDoubling}, where a condition under which graphs embedded in a doubling metric space can be made navigable is identified. The same idea was explored in more general spaces by \cite{Fraign}. Again, the main difference in our approach to small world network design lies in considering heterogeneous demand, an aspect of small-world networks not investigated in earlier work. 

The relationship between the small-world network design and content search has been also observed in earlier work \cite{Lifshits08} and was exploited by \cite{Lifshits09} in proposing their data structures for content search through comparisons; we further expand on this issue in Section~\ref{section:relationship}, as this is an approach we also follow.


%
%

\section{Definitions and Notation}\label{section:definitions}
In this section we introduce some definitions and notation which will be used throughout this paper.
\subsection{Objects and Metric Embedding}
Consider a set of objects $\Nn$, where $|\myset{N}|=n$. We assume that there exists a metric space $(\myset{M}$,$d$),  
where $d(x,y)$ denotes the distance between $x,y\in \myset{M}$, 
such that objects in $\Nn$ are embedded in   $(\myset{M}$,$d$): \emph{i.e.}, there exists a one-to-one mapping from $\myset{N}$ to a subset of $\myset{M}$. 

\sloppy The objects in $\Nn$ may represent, for example, pictures in a database. The metric embedding can be thought of as a mapping of the database entries to a set of features (\emph{e.g.}, the age of person depicted, her hair and eye color, \emph{etc.}). The distance between two objects would then capture how ``similar'' two objects are w.r.t.~these features. In what follows, we will abuse notation and write $\myset{N}\subseteq\myset{M}$, keeping in mind that there might be difference between the physical objects (the pictures) and their embedding (the attributes that characterize them).
\fussy

Given an object $z\in \Nn$, we can order objects according to their distance from $z$. We will write $x \lsteq{z} y$ if $d(x,z)\leq d(y,z)$. 
 Moreover, we will write $x\eqvl{z}y$  if $d(x,z)=d(y,z)$ and $x\lst{z} y$ if $x \lsteq{z} y$ but not $x\eqvl{z}y$. Note that $\eqvl{z}$ is an equivalence relation, and hence partitions $\Nn$ into equivalence classes. Moreover, $\lsteq{z}$ defines a total order over these  equivalence classes, with respect to their distance from $z$. 
Given a non-empty set $A\subseteq \Nn$, we denote by $\min_{\lsteq{z}}A$ the object in $A$ closest to $z$, \emph{i.e.}
$$\min_{\lsteq{z}}A=w\in A ~ \text{s.t.}~w\lsteq{z}v$$ for all $v\in A$.

\subsection{Comparison Oracle}\label{section:oracle}
A \emph{comparison oracle} \cite{Lifshits08} is an oracle that, given two objects $x,y$ and a target $t$, returns the closest object to $t$. More formally, 
\begin{equation}\label{oracle}
\text{\oracle}(x,y,t)=\left\{
\begin{array}{ll} 
x & \text{if } x\lst{t} y,\\ 
y & \text{if } x\grt{t} y,\\
x\text{ or }y &  \text{if } x\eqvl{t} y.
\end{array} \right.
\end{equation}
Observe that if $x=\oracle(x,y,t)$ then $x\lsteq{t} y$; this does not necessarily imply however that $x\lst{t}y$. 

This oracle basically aims to capture the behavior of human users. A human interested in locating, \emph{e.g.}, a target picture $t$ within the database, may be able to compare other pictures with respect to their similarity to this target but  cannot associate a numerical value to this similarity. Moreover, when the pair of pictures compared are equally similar to the target, the decision made by the human may be arbitrary.

\sloppy It is important to note here that although we write $\oracle(x,y,t)$ to stress that a query always takes place with respect to some target $t$, in practice the target is hidden and only known by the oracle. Alternatively, following the ``oracle as human'' analogy, the human user has a target in mind and uses it to compare the two objects, but never discloses it until actually being presented with it.

\fussy
 Note that our oracle is weaker than one that correctly identifies the relationship $x\eqvl{t}y$ and, \emph{e.g.}, returns a special character ``='' once two such objects are proposed: to see this, observe that oracle \eqref{oracle} can be implemented by using this stronger oracle. Hence, all our results hold if we are provided with such an oracle instead.

\subsection{Demand}

We denote by $\Nn\times \Nn$ the set of all ordered pairs of objects in $\Nn$.  For $(s,t)\in \Nn\times \Nn$, we will call $s$ the \emph{source} and $t$ the \emph{target} of the ordered pair. 
We will consider a probability distribution $\lambda$ over all ordered pairs of objects in $\Nn$ which we will call the \emph{demand}.
In other words,  $\lambda$ will be a non-negative function  
such that $$\sum_{(s,t)\in\Nn\times\Nn} \lambda(s,t)=1.$$  
In general, the demand can be \emph{heterogeneous} as $\lambda(s,t)$ may vary across different sources and targets.
We refer to the marginal distributions $$\nu(s)=\sum_{t}\lambda(s,t), \quad \mu(t)=\sum_{s}\lambda(s,t),$$ as the \emph{source} and \emph{target} distributions, respectively. Moreover, will refer to the support of the target distribution 
$$ \myset{T}=\supp(\mu)=\{x\in \Nn:\text{ s.t. }\mu(x) > 0\}$$
 as the \emph{target set} of the demand. 

As we will see in Section~\ref{section:mainresults}, the target distribution $\mu$ will play an important role in our analysis. In particular, two quantities that affect the performance of searching in our scheme will be the \emph{entropy} and the \emph{doubling constant} of the target distribution. We introduce these two notions formally below.

\subsection{Entropy}\label{section:entropy}

Let $\sigma$ be a probability distribution over $\Nn$. 
The \emph{entropy} of $\sigma$ is defined as
\begin{align}H(\sigma)=\sum_{x\in \supp(\sigma)} \sigma(x)\log\frac{1}{\sigma(x)}.\label{entropy}\end{align}
We define the \emph{max-entropy} of $\sigma$ as 
\begin{align}H_{\max}(\sigma) = \max_{x\in \supp(\sigma)} \log \frac{1}{\sigma(x)} \label{maxentropy}.\end{align} 

The entropy has strong connections with the content search problem. More specifically, suppose that we have access to a  so-called \textit{membership oracle} \cite{Cover} that can answer  queries of the following form: 
\begin{quote}
``Given a target $t$ and a subset $A\subseteq \Nn$, does $t$ belong to $A$?''\end{quote}
 Assume now that an object $t$ is selected according to a distribution $\mu$. 
 It is well known that to find a target $t$ one needs to submit at least $H(\mu)$ queries, on average,  to the oracle described above (see, chap. 2, \cite{Cover}). Moreover, there exists an algorithm (Huffman coding) 
that finds the target with only $H(\mu)\!+\!1$ queries on average \cite{Cover}. In the worst case, which occurs when the target is the least frequently selected object, the algorithm requires  $H_{\max}(\mu)\!+\!1$ queries to identify $t$. 

Our work identifies similar bounds assuming that one only has access to a comparison oracle, like the one described by \eqref{oracle}. Not surprisingly, the entropy of the target distribution $H(\mu)$ shows up in the performance bounds that we obtain (Theorems~\ref{th:content-search} and \ref{lowerbound}). However, searching for an object will depend not only on the entropy of the target distribution, but also on the topology of the target set $\myset{T}$. This will be captured by the doubling constant of $\mu$, which we describe in more detail below.

\subsection{Doubling Constant} \label{sectiondoubling}
Given an object $x\in \Nn$, we denote by
\begin{align}\label{ball}
B_x(r) = \{y\in \myset{M}: d(x,y)\leq r\}
\end{align}
the closed ball of radius $r\geq 0$ around $x$.
Given a probability distribution $\sigma$ over $\Nn$ and a set $A\subset\Nn$ let $$\sigma(A)=\sum_{x\in A} \sigma(x).$$
We define the \emph{doubling constant} $c(\sigma)$ of a distribution $\sigma$ to be the minimum $c>0$ for which 
\begin{align}\label{doubling}\sigma(B_x(2r)) \leq c \cdot \sigma(B_x(r)),\end{align}
for any $x\in \supp(\sigma)$ and any $r\geq 0$.
Moreover, will say that $\sigma$ is $c$-\emph{doubling} 
if $c(\mu)=c$.

\begin{figure}[!t]
\begin{center}

\input{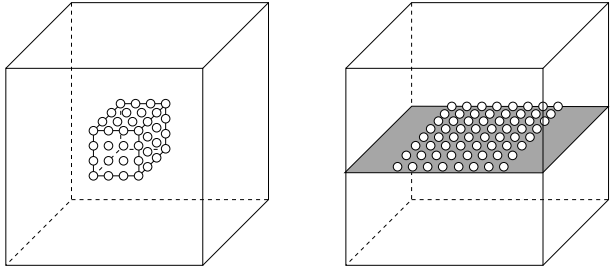}
\caption{Example of dependence of $c(\sigma)$ on the topology of the support $\supp(\sigma)$. When $\supp(\sigma)$ consists of $n=64$ objects arranged in a cube, $c(\sigma)=2^3$. If, on the other hand, these $n$ objects are placed on a plane, $c(\sigma) = 2^2$. In both cases $\sigma$ is assumed to be uniform, and $H(\sigma) = \log N$}\label{example}
\end{center}
\end{figure}

Note that, contrary to the entropy $H(\sigma)$, 
 the doubling constant $c(\sigma)$ depends on the topology of $\supp(\sigma)$, determined by the embedding of $\Nn$ in  the metric space $(\myset{M},d)$. This is illustrated in Fig.~\ref{example}. In this example, $|\Nn|=64$, and the set $\Nn$ is embedded in a 3-dimensional cube. Assume that $\sigma$ is the uniform distribution over the $N$ objects; if these objects are arranged uniformly in a cube, then  $c(\sigma)=2^3$; if however these $n$ objects are arranged uniformly in a 2-dimensional plane, $c(\sigma) = 2^2$. Note that, in contrast, the entropy  of $\sigma$ in both cases equals  $\log n$ (and so does the max-entropy).

\begin{table*}[t!]
\begin{center}
\caption{Summary of Notation}
\begin{tabular}{|l|l|l|l|}
\hline
$\myset{N}$ & Set of objects & $(\myset{M},d)$ & Metric space\\
$d(x,y)$ & Distance between $x,y\in\myset{M}$ & $x\lsteq{z}y$ & Ordering w.r.t. distance from $z$\\
$x\lsteq{z}y$ & Ordering w.r.t. distance from $z$ & $x\eqvl{z}y$ & $x$ and $y$ at same distance from $z$\\
$\lambda$ & The demand distribution & $\nu$ & The source distribution \\
$\mu$ & The target distribution & $\myset{T}$ & The target set\\
$H(\sigma)$ & The entropy of  $\sigma$ & $H_{\max}(\sigma)$ & The max-entropy of $\sigma$\\
$B_x(r)$ & The ball of radius $r$ centered at $x$ &$c(\sigma)$ & The doubling constant of $\sigma$\\
$\myset{S}$ & The set of shortcut edges &  $\myset{L}$ & The set of local edges\\
$\bar{C}_{\myset{S}}$ & Expected cost of greedy forwarding given set $\myset{S}$ & $\bar{C}_{\Ff}$ & Expected search cost of policy $\Ff$\\
\hline
\end{tabular}
\end{center}
\end{table*}

\section{Problem Statement}\label{section:problemstatement}
We now formally define the two problems that will be the main focus of this paper. The first is the problem of \emph{content search through comparisons} and the second is the \emph{small-world network design} problem.

\subsection{Content Search Through Comparisons}\label{section:contentsearch}

For the content search problem, we consider the object set $\Nn$, embedded in $(\myset{M},d)$. Although this embedding exists, 
we are constrained by not being able to directly compute object distances. Instead, we only have access to a comparison oracle, like the one defined in Section~\ref{section:oracle}.

Given access to the above oracle, we would like to navigate through $\Nn$ until we find a target object. 
In particular, we define  \emph{greedy content search} as follows. Let $t$ be the target object and $s$ some object that serves as a starting point. The greedy content search algorithm proposes an object $w$ and asks the oracle to select, between $s$ and $w$, the object  closest to the target $t$, \emph{i.e.}, it evokes $\oracle(s,w,t)$. This process is repeated until the oracle returns something other than $s$, \emph{i.e.}, the proposed object  is ``more similar'' to the target $t$. Once this happens, say at the proposal of some $w'$, if $w'\neq t$, the greedy content search repeats the same process now from $w'$. 
If at any point the proposed object is $t$, the process terminates.

Recall that in the ``oracle as a human'' analogy the human cannot reveal $t$ before actually being presented with it. We similarly assume here that $t$ is never ``revealed'' before actually being presented to the oracle. Though we write $\text{\oracle}(x,y,t)$ to stress that the submitted query is w.r.t.~proximity to $t$, the target $t$ is not a priori known. In particular, as we see below, the decision of which objects $x$ and $y$ to present to the oracle cannot directly depend on $t$.

More formally, let $x_k,y_k$ be the $k$-th pair of objects submitted to the oracle: $x_k$ is the \emph{current object}, which greedy content search is trying to improve upon, and $y_k$ is the \emph{proposed object}, submitted to the oracle for comparison with $x_k$. Let $$o_k=\oracle(x_k,y_k,t)\in \{x_k,y_k\}$$ be the oracle's response, and define  
$$\history_k=\{(x_i,y_i,o_i)\}_{i=1}^k,\quad k=1,2,\ldots$$ 
be the sequence of the first $k$ inputs given to the oracle, as well as the responses obtained; $\history_k$ is the ``history'' of the content search up to and including the $k$-th access to the oracle. 

The source object is always one of the first two objects submitted to the oracle, \emph{i.e.}, $x_1=s$. Moreover,
in greedy content search, 
$$x_{k+1}=o_k, \quad k=1,2,\ldots$$ 
\emph{i.e.}, the current object is always the closest to the target among the ones submitted so far.

On the other hand, the selection of the proposed object $y_{k+1}$ will be determined by the history $\history_k$ and the object $x_k$. In particular, given $\history_k$ and the current object $x_k$ there exists a mapping $(\history_k,x_k)\mapsto \myset{F}(\history_k,x_k)\in \Nn$ such that
\begin{align*}y_{k+1}=\myset{F}(\history_k, x_k),\quad k=0,1,\ldots,\end{align*} 
where here we take $x_0=s\in \Nn$ (the source/starting object) and $\history_0=\emptyset$ (\emph{i.e.}, before any comparison takes place, there is no history). 

We will call the mapping $\myset{F}$ the \emph{selection policy} of the greedy content search. In general, we will allow the selection policy to be randomized; in this case, the object returned by $\myset{F}(\history_k,x_k)$ will be a random variable, whose distribution 
\begin{align}\prob(\myset{F}(\history_k,x_k)=w), \quad w\in \Nn,\label{mapdistr}\end{align} is fully determined by $(\history_k,x_k)$.
Observe that $\myset{F}$  depends on the target $t$  only indirectly, through $\history_k$ and $x_k$; this is consistent with our assumption that $t$ is only ``revealed'' when it is eventually located.

We will say that a selection policy is \emph{memoryless} if it depends on $x_k$ but not on the history $\history_k$. 
In other words, the distribution \eqref{mapdistr} is the same when  $x_k=x\in \Nn$, irrespectively of the comparisons performed prior to reaching $x_k$.


Our goal is to select  $\myset{F}$ so that we minimize the number of accesses to the oracle.
In particular, given a source object $s$, a target $t$ and a selection policy $\myset{F}$, we define the search cost $$C_{\myset{F}}(s,t) = \inf\{k:y_k=t\}$$ to be the number of proposals to the oracle until $t$ is found. This is a random variable, as $\myset{F}$ is randomized; let $\expect[C_{\myset{F}}(s,t)]$ be its expectation. The Content Search Through Comparisons problem is then defined as follows:

\begin{quote}
\textsc{ Content Search Through Comparisons (CSTC)}:
Given an embedding of $\Nn$ into $(\myset{M},d)$ and a demand distribution $\lambda(s,t)$,  select $\myset{F}$ that minimizes  the expected search cost
$$\bar{C}_{\Ff}= \sum_{(s,t)\in \Nn\times\Nn}\lambda(s,t)\expect[C_{\myset{F}}(s,t)].$$
\end{quote}
Note that, as $\mathcal{F}$ is randomized, the free variable in the above optimization problem is the distribution \eqref{mapdistr}.


\subsection{Small-World Network Design}\label{section:smallworld}
In the small network design problem, we again consider the objects in $\Nn$, embedded in $(\myset{M},d)$. It is now assumed however that the objects in $\Nn$ are connected to each other. The network formed by such connections is represented by a directed graph $G(\Nn,\myset{L}\cup \myset{S})$, where $\myset{L}$ is the set of \emph{local} edges and $\myset{S}$ is the set of \emph{shortcut} edges. These edge sets are disjoint, \emph{i.e.}, $\myset{L}\cap \myset{S}=\emptyset$.

The edges in $\myset{L}$ are typically assumed to satisfy the following property:
\begin{property}\label{greedyprop}
For every pair of distinct objects $x,t\in \Nn$  there exists an object $u$
 adjacent to $x$  such that $(x,u)\in \myset{L}$  and $u\lst{t}x$.
\end{property}
In other words, for any object $x$ and a target $t$, $x$ has a local edge leading to an object closer to $t$.

Recall that in the content search problem the goal was to find $t$  (starting from source $s$) using only accesses to a comparison oracle. Here the goal is to use such an oracle to route a message from $s$ to  $t$ over the links in graph $G$. In particular,
given graph $G$, we define  \emph{greedy forwarding} \cite{kleinberg} over $G$ as follows.
Let $\Gamma(s)$ be the neighborhood of $s$, \emph{i.e.},
$$ \Gamma(s) = \{u\in \Nn\text{ s.t. }(s,u)\in \myset{L}\cup \myset{S} \}.$$
 Given a source  $s$ and a target $t$, greedy forwarding sends a message to neighbor $w$ of $s$ that is as close to $t$ as possible, \emph{i.e.},
\begin{align}w = \textstyle{\min_{\lsteq{t}}} \Gamma(s).\label{closest}\end{align} If $w\neq t$, the above process is repeated at $w$; if $w=t$, greedy forwarding terminates. 

Note that local edges, through Property~\ref{greedyprop}, guarantee that greedy forwarding from any source $s$ will eventually reach $t$: there will always be a neighbor that is closer to $t$ than the object currently having the message.  Moreover, the closest neighbour $w$ selected through \eqref{closest} can be found using a comparison oracle. In particular, if the message is at an object $x$, $|\Gamma(x)|$ queries to the oracle will suffice to find  the neighbor that is closest to the target.  

%

The edges in $\myset{L}$ are typically called ``local'' because they are usually 
 determined by object proximity.
For example, in the classical paper by Kleinberg \cite{kleinberg}, objects are arranged uniformly in a rectangular $k$-dimensional grid---with no gaps---and $d$ is taken to be the Manhattan distance on the grid. 
 Moreover, there exists an $r\geq 1$ such that any two objects at distance less than $r$ have an edge in $\myset{L}$. In other words, 
\begin{align}\myset{L}=\{(x,y)\in \Nn\times\Nn \text{ s.t. }d(x,y)\leq r\}.\label{local}\end{align}
Assuming every position in the rectangular grid is occupied, such edges indeed satisfy Property~\ref{greedyprop}.  In this work, we will not require that edges in $\myset{L}$ are given by \eqref{local} or some other locality-based definition; our only assumption  is  that they satisfy Property~\ref{greedyprop}. Nevertheless, for the sake of consistency with prior work, we  also refer to edges in $\myset{L}$ as ``local''.

The shortcut edges $\myset{S}$ need not satisfy Property~\ref{greedyprop}; our goal is to select these shortcut edges in a way so that greedy forwarding is as efficient as possible. 

In particular, we assume that we can select no more than $\beta$ shortcut edges, where $\beta$ is a positive integer. For $S$ a subset of $\Nn\times\Nn$ such that $|S|\leq \beta$, we denote by $C_S(s,t)$ the cost of greedy forwarding, in message hops, for forwarding a message from $s$ to $t$ given that $\mathcal{S}=S$. We allow the selection of shortcut edges to be random: the set $\myset{S}$ can be a random variable over all subsets $S$ of $\Nn\times\Nn$ such that $|S|\leq \beta$.

 We denote by
\begin{align}\label{Sdistr}\Pr(\myset{S}=S), \quad S\subseteq \Nn\times\Nn\text{ s.t. }|S|\leq \beta \end{align}
the distribution of $\myset{S}$.
Given a source  $s$ and a target $t$, let $$\expect[C_{\myset{S}}(s,t)]= \sum_{S \subseteq \Nn\times\Nn: |S|\leq \beta} C_{S}(s,t) \cdot \prob(\myset{S}=S)$$ be the expected cost of forwarding a message from $s$ to $t$ with greedy forwarding, in message hops. 

We  consider again a heterogeneous demand: a source and target object are selected at random from $\Nn\times \Nn$ according to a
demand probability distribution $\lambda$. 
 The \emph{small-world network design problem} can then be formulated as follows.

\begin{quote}
\textsc{Small-World Network Design (SWND)}:
Given an embedding of $\Nn$ into $(\myset{M},d)$, a set of local edges $\myset{L}$, a demand distribution $\lambda$, and an integer $\beta>0$, select a r.v. $\myset{S}\subset \Nn\times\Nn$ that minimizes 
$$\bar{C}_{\myset{S}}= \sum_{(s,t)\in \Nn\times\Nn}\lambda(s,t)\expect[C_{\myset{S}}(s,t)]$$ subject to $|\myset{S}|\leq\beta$.
\end{quote}
In other words, we wish to select $\myset{S}$ so that the cost of greedy forwarding is minimized. Note that, since $\myset{S}$ is a random variable, the free variable of the above optimization problem is essentially the distribution of $\myset{S}$, given by \eqref{Sdistr}. 

\subsection{Relationship Between \textsc{SWND} and \textsc{CSTC} }\label{section:relationship}
In what follows, we  try to give some intuition about how SWND and CSTC are related and why the upper bounds we obtain for these two problems are identical, without resorting to the technical details appearing in our proofs.

Consider the following version of the SWND problem, in which we place three additional restrictions to the selection of the shortcut edges. 
First, $|\myset{S}|=n$, \emph{i.e.}, we can only select $n=|\Nn|$ shortcut edges.  Second, for every $x\in \Nn$,  there exists exactly one directed edge $(x,y)\in \myset{S}$: each object has exactly one out-going edge incident to it.  Third, the object $y$ to which object $x$ connects to is selected independently at each $x$, according to a probability distribution $\ell_x(y)$. In other words, for $\Nn=\{x_1,x_2,\ldots,x_n\}$, the joint distribution of shortcut edges has the form:
\begin{align}\Pr(\myset{S} = \{(x_1,y_1),\ldots(x_n,y_n)\}) = \prod_{i=1}^{n}\ell_{x_i}(y_i). \label{restricted}\end{align}
We  call this version of the \textsc{SWND} problem the \emph{one edge per object} version, and denote it by \textsc{1-SWND}. Note that, in \textsc{1-SWND}, the free variables are the distributions $\ell_x$, $x\in \Nn$, which are to be selected in order to minimize the average cost $\bar{C}_{\myset{S}}$.

\begin{figure}[t!]
\begin{center}
\hspace{.5cm}
\hspace*{-0.5cm}\input{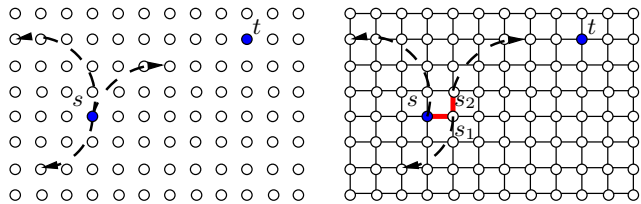}
\caption{An illustration of the relationship between 1-SWND and CTSC. In CTSC, the source $s$ samples objects independently from the same distribution until it locates an object closest to the target $t$. In 1-SWND, the re-sampling is emulated by the movement to new neighbors. Each neighbor ``samples'' a new object independently, from a slightly perturbed distribution, until one closest to the target $t$ is found.}\label{cost-diff}
\end{center}
\end{figure} 
Consider now following content selection policy for \textsc{CTSC}:
$$\Pr(\myset{F}(x_k)=w) = \ell_{x_k}(w), \quad\text{ for all }w\in \Nn$$
In other words, if the proposed object at $x_k$ is sampled according to the same distribution as the shortcut edge in \text{1-SWND}. This selection policy is memoryless as it does not depend on the history $\history_k$ of objects presented to the oracle so far.

A parallel between these two problems can be drawn as follows. Suppose that the same source/target pair $(s,t)$ is given in both problems. In content search, while starting from node $s$, the memoryless selection policy  draws independent samples from distribution $\ell_s$ until an object closer to the target than $s$ is found. 

In contrast, greedy forwarding in 1-SWND can be described as follows. Since shorcut edges are generated independently, we can assume that they are generated while the message is being forwarded. Then, greedy forwarding at the source object can be seen as sampling an object from distribution $\ell_s$, namely, the one incident to its shortcut edge. If this object is not closer to the target than $s$, the message is forwarded to a neigboring node $s_1$ over a local edge of $s$. Node $s_1$  then samples independently a node from distribution $\ell_{s_1}$ this time---the one incident to its shorcut edge.

Suppose that the distributions $\ell_x$ vary only slightly across neighboring nodes. Then, forwarding over local edges corresponds to the independent resampling occuring in the content search problem. Each move to a new neighbor samples a new object (the one incident to its shortcut edge) independently of previous objects but from a slightly perturbed distribution. This is repeated until an object closer to the target $t$ is found, at which point the message  moves to a new neighborhood over the shortcut edge. 


Effectively,  re-sampling  is ``emulated'' in 1-SWND  by the movement to new neighbors. 
This is, of course, an informal argument; we refer the interested reader to the proofs of Theorems~\ref{th:small-world} and Theorem~\ref{th:content-search} for a  rigorous statement of the relationship between the two problems.

\section{Main Results}\label{section:mainresults}
We now present our main results with respect to \textsc{SWND} and \textsc{CSTC}.
 Our first result is negative:  optimizing greedy forwarding  is a hard problem.  
\begin{theorem}\label{hardnessthm}
SWND is NP-hard.
\end{theorem}
The proof of this theorem can be found in Section \ref{proof:hardness}. 
In short, the proof  reduces  \textsc{DominatingSet} to the decision version of \textsc{SWND}. Interestingly, the reduction is to a $\textsc{SWND}$ instance in which (a) the metric space is a 2-dimensional grid, (b) the distance metric is the Manhattan distance on the grid and (c) the local edges are given by \eqref{local}. Thus, \textsc{SWND} remains NP-hard even in the original setup considered by Kleinberg~\cite{kleinberg}.

The NP-hardness of \textsc{SWND} suggests that this problem cannot be solved in its full generality. Motivated by this, as well as its relationship to content search through comparisons,  we consider below the restricted version \textsc{1-SWND}. In particular, we provide a distribution of edges for \textsc{1-SWND} for which an upper-bound of search cost exists. This upper-bound can be expressed in terms of the entropy and the doubling dimension of the target distribution $\mu$. Through the relationship of \textsc{1-SWND} with \textsc{CSTC}, we are able to obtain a greedy  content search strategy whose cost can also be bounded the same way.



For a given demand $\lambda$, recall that $\mu$ is the marginal distribution of the demand $\lambda$ over the target set $\myset{T}$, and that 
for $A\subset \Nn$,  $\mu(A)=\sum_{x\in A} \mu(x)$. 
Then, for any two objects $x,y \in \Nn$, we define the \textit{rank} of object $y$ w.r.t.~object $x$ as follows:
\begin{equation}\label{rank}
r_x(y)\equiv\mu(B_x(d(x,y)))
\end{equation}
where $B_x(r)$ is the closed ball with radius $r$ centered at $x$.


Suppose now that shortcut edges are generated according to the joint distribution \eqref{restricted}, where 
the  outgoing  link from an object $x\in \Nn$ is selected according to the following probability:
\begin{align}\label{shortcutdistr}
\ell_x(y)\propto \frac{\mu(y)}{r_x(y)},
\end{align}
for $y\in \supp(\mu)$, while  for $y\notin \supp(\mu)$ we define $\ell_x(y)$ to be zero. 
Eq.~\eqref{shortcutdistr} implies the following appealing properties. 
\begin{itemize}
\item For two objects $y,z$ that have the same distance from $x$, if $\mu(y)>\mu(z)$  then $\ell_x(y)>\ell_x(z)$, \emph{i.e.}, y has a higher probability of being connected to $x$.
\item When two objects $y,z$ are equally likely to be targets, if $y\lst{x} z$ then $\ell_x(y)>\ell_x(z)$.
\end{itemize}
The distribution \eqref{shortcutdistr} thus biases both towards objects close to $x$ as well as towards objects that are likely to be targets. 
Finally, if the metric space $(\myset{M},d)$ is a $k$-dimensional grid and the targets are uniformly distributed over $\Nn$ then $\ell_x(y)\propto (d(x,y))^{-k}.$ 
This is the shortcut distribution used by \cite{kleinberg}; Eq~\eqref{shortcutdistr} is thus  a generalization of this distribution to heterogeneous targets as well as to more general metric spaces. 

 Our next theorem, whose proof is in Section~\ref{proof:small-world},
 relates the cost of greedy forwarding under \eqref{shortcutdistr}  to the entropy $H$, the max-entropy $H_{\max}$ and the doubling parameter  $c$ of the target distribution $\mu$.
\begin{theorem}\label{th:small-world}
Given a demand $\lambda$, consider the set of shortcut edges $\myset{S}$ sampled according to \eqref{restricted}, where $\ell_x(y)$, $x,y\in \Nn$, are given by \eqref{shortcutdistr}. Then 
$$\bar{C}_{\Ss}\leq 6 c^3(\mu) \cdot H(\mu)\cdot  H_{\max}(\mu).$$
\end{theorem}
Note that the bound in Theorem~\ref{th:small-world} depends  on $\lambda$ only through the target distribution $\mu$. In particular, it holds for \emph{any} source distribution $\nu$, and \emph{does not require} that sources are selected independently of the targets $t$. 
Moreover, if $\Nn$ is a $k$-dimensional grid and $\mu$ is the uniform distribution over $\Nn$, the above bound becomes $O(\log^2n)$, retrieving thus the result of~\cite{kleinberg}.

\begin{algorithm}[t]
\caption{Memoryless Content Search}
\label{algo:memoryless}
\begin{algorithmic}[1]
\REQUIRE{\oracle($\cdot$,$\cdot$,$t$) , demand distribution $\mu$, starting object $s$.}
\ENSURE{target $t$.}
\STATE $x \gets s$ 
\WHILE{$x\neq t$}
\STATE Sample $y\in \Nn$  from the probability distribution 
\begin{displaymath}
\prob(\myset{F}(\history_k,x_k)=y)=\ell_{x_k}(y).
\end{displaymath}
\STATE $x\gets \oracle(x,y,t)$.
\ENDWHILE
\end{algorithmic}
\end{algorithm}
Exploiting an underlying relationship between \textsc{1-SWND} and \textsc{CSTC}, we can obtain an efficient selection policy for greedy content search. In particular, 
\begin{theorem} \label{th:content-search} Given a demand $\lambda$, consider the memoryless selection policy $\myset{F}$ outlined in Algorithm~\ref{algo:memoryless}. Then 
$$\bar{C}_{\Ff}\leq 6 c^3(\mu) \cdot H(\mu)\cdot  H_{\max}(\mu).$$
\end{theorem}
The proof of this theorem is given in Section~\ref{proof:content-search}.
 Like Theorem~\ref{th:small-world}, Theorem~\ref{th:content-search}  characterises the search cost in terms of the doubling constant, the entropy and the max-entropy of $\mu$. This is very appealing, given (a) the relationship between $c(\mu)$ and the topology of the target set and (b) the classic result regarding the entropy and accesses to a membership oracle, as outlined in Section~\ref{section:definitions}. 

The distributions $\ell_x$ are defined in terms of the  embedding of $\myset{N}$ in $(\myset{M},d)$ and the target distribution $\mu$.  Interestingly, however, the bounds of Theorem~\ref{th:content-search} can be achieved if \emph{neither} the  embedding in $(\myset{M},d)$ \emph{nor} the target distribution $\mu$ are a priori known. In our technical report~\cite{techrep} we propose an adaptive algorithm that asymptotically achieves the performance guarantees of  Theorem~\ref{th:content-search} only through access to a comparison oracle.  In short, the algorithm learns the ranks $r_x(y)$ and the target distribution $\mu$ as searches through comparisons take place.

A question arising from Theorems~\ref{th:small-world} and \ref{th:content-search} is how tight these bounds are. Intuitively, we expect that the optimal shortcut set $\myset{S}$ and the optimal selection policy $\myset{F}$ depend both on the entropy of the target distribution and on its doubling constant. 
Our next theorem, whose proof is in Section~\ref{proofoflowerbound}, 
establishes that this is the case for $\myset{F}$.
\begin{theorem}\label{lowerbound}
For any integer $K$ and $D$, there exists a metric space $(\myset{M},d)$ and a target measure $\mu$ with entropy $H(\mu)=K\log(D)$ and doubling constant $c(\mu)=D$ such that the  average search cost of any selection policy $\myset{F}$ satisfies
\begin{equation}\label{lower}
\bar{C}_{\myset{F}}\ge H(\mu)\frac{c(\mu)-1}{2\log(c(\mu))}\cdot
\end{equation}
\end{theorem}
Hence, the bound in Theorem~\ref{th:content-search} is tight within a $c^2(\mu)\log(c(\mu))H_{\max}$ factor. 

\section{Learning Algorithm}\label{section:learning}
Section~\ref{section:mainresults} established bounds on the cost of greedy content search provided that the  distribution \eqref{shortcutdistr} is used to propose items to the oracle. Hence, if the  embedding of $\myset{N}$ in $(\myset{M},d)$ and target distribution $\mu$  are known, it is possible to perform greedy content search with the performance guarantees provided by Theorem~\ref{th:content-search}.

 In this section, we turn our attention to how such bounds can be achieved if \emph{neither} the  embedding in $(\myset{M},d)$ \emph{nor} the target distribution $\mu$ are a priori known. To this end, we propose a novel adaptive algorithm that achieves the performance guarantees of  Theorem~\ref{th:content-search} without access to the above information.

Our algorithm effectively learns the ranks $r_x(y)$ of objects and the target distribution $\mu$ as time progresses. 
It does not require that distances between objects are at any point disclosed; instead, we assume that it only has access to a comparison oracle, slightly stronger than the one described in Section~\ref{section:smallworld}. 

It is important to note that our algorithm is \emph{adaptive}: though we prove its convergence under a stationary regime, the algorithm can operate in a dynamic environment. For example, new objects can be added to the database while old ones can be removed. Moreover, the popularity of objects can change as time progresses. Provided that such changes happen infrequently, at a larger timescale compared to the timescale in which database queries are submitted, our algorithm will be able to adapt and  converge to the desired behavior.

\subsection{Demand Model and Probabilistic Oracle}
We  assume that time is slotted and that at each timeslot $\tau=0,1,\ldots$ a new query is generated in the database. As before, we assume that the source and target of the new query are selected according to a demand distribution $\lambda$ over $\Nn\times \Nn$. We again denote by $\nu$, $\mu$ the (marginal) source and target distributions, respectively. 

Our algorithm will require that the support of both the source and target distributions is $\Nn$, and more precisely that 
\begin{align}\label{strong}\lambda(x,y)>0, \text{ for all }x,y\in \Nn.\end{align}
The requirement that the target set $\myset{T}=\supp(\mu)$  is $\Nn$ is necessary to ensure learning; we can only infer the relative order w.r.t.~objects $t$ for which questions of the form $\oracle(x,y,t)$ are submitted to the oracle.
Moreover, it  is natural in our model to assume that the source distribution $\nu$ is at the discretion of our algorithm: we can choose which objects to propose first to the user/oracle. In this sense, for a given target distribution $\mu$ s.t.~$\supp(\mu)=\Nn$, \eqref{strong} can be enforced, \emph{e.g.}, by selecting source objects uniformly at random from $\Nn$ and independently of the target. 

We consider a slightly stronger oracle than the one described in Section~\ref{section:contentsearch}. In particular, we again assume that
\begin{equation}
\text{\oracle}(x,y,t)=\left\{
\begin{array}{ll} 
x & \text{if } x\lst{t} y,\\ 
y & \text{if } x\grt{t} y.
\end{array} \right.
\end{equation}
However, we further assume that if $x\eqvl{t}y$, then $\oracle(x,y,t)$ can return either of the two possible outcomes \emph{with non-zero probability}. This is stronger than the oracle in  Section~\ref{section:contentsearch}, where we assumed that the outcome will be arbitrary. 
We should point out here that this is still weaker than an oracle that correctly identifies  $x\eqvl{t} y$ (\emph{i.e.}, the human states that these objects are at equal distance from $t$) as, given such an oracle, we can implement the above probabilistic oracle by simply returning $x$ or $y$ with equal probability.


\subsection{Data Structures}
For every object $x\in\Nn$, the database storing $x$ also maintains the following associated data structures. The first data structure is a counter keeping track of how often the object $x$ has been requested so far. The second data structure maintains an order of the objects in $\Nn$; at any point in time, this total order is an ``estimator'' of $\lsteq{x}$, the order of objects with respect to their distance from $x$. We describe each one of these two data structures in more detail below.

\paragraph{Estimating the Target Distribution}

The first data structure associated with an object $x$ is an estimator of $\mu(x)$, \emph{i.e.}, the probability with which $x$ is selected as a target. A simple method for keeping track of this information is through a counter $C_x$. This counter $C_x$ is initially set to zero and is incremented every time  object $x$ is the target. If $C_x(\tau)$ is the counter at timeslot $\tau$, then \begin{align}\hat{\mu}(x) =C_x(\tau)/\tau \label{muest}\end{align} is  an unbiased estimator of $\mu(x)$. To avoid counting to infinity 
a ``moving average'' (\emph{e.g.}, and exponentially weighted moving average) could be used instead.

\paragraph{Maintaining a Partial Order}
The second data structure $\myset{O}_x$ associated with each $x\in\Nn$ maintains a total order of objects in $\Nn$ w.r.t.~their similarity to $x$. 
It supports an operation called \order() that returns a partition of objects in $\Nn$ along with a total order over this partition. In particular, the output of $\myset{O}_x.\order()$ consists of an ordered sequence of disjoint sets $A_1, A_2,\ldots, A_j$, where $\bigcup A_i=\Nn\setminus\{x\}$. Intuitively, any two objects in a set $A_i$ are considered to be at equal distance from $x$, while among two objects $u\in A_i$ and $v\in A_j$ with $i<j$ the object $u$ is assumed to be the closer to $x$.

\sloppy
Moreover, every time that the algorithm evokes $\oracle(u,v,x)$, and learns, \emph{e.g.}, that $u\lsteq{x} v$, the data structure $\myset{O}_x$ should be updated to reflect this information. In particular, if the algorithm has learned so far the order relationships \begin{align}u_1\lsteq{x} v_1,\quad u_2\lsteq{x} v_2,\quad \ldots, \quad u_i\lsteq{x} v_i\label{learnedorder}\end{align}
 $\myset{O}_x$.\order() should return the objects in $\Nn$ sorted in such a way that all relationships in \eqref{learnedorder} are respected. In particular, object $u_1$ should appear before $v_1$, $u_2$ before $v_2$, and so forth. To that effect, the data structure should also support an operation called $\myset{O}_x$.\add($u$,$v$) that adds the order relationship $u \lsteq{x} v$ to the constraints respected by the output of $\myset{O}_x$.\order().
\fussy

A simple (but not the most efficient) way of implementing this data structure is to represent order relationships through a directed acyclic graph. Initially, the graph's vertex set is $\Nn$ and its edge set is empty. Every time an operation \add($u$,$v$)  is executed, an edge is added between vertices $u$ and $v$. If the addition of the new edge creates a cycle then all nodes in the cycle are collapsed to a single node, keeping thus the graph acyclic. Note that the creation of a cycle $u\to v\to\ldots\to w\to u$ implies that $u\eqvl{x}v\eqvl{x}\ldots\eqvl{x}w$, \emph{i.e.}, all these nodes are at equal distance from $x$. 

Cycles can be detected by using depth-first search over the DAG \cite{clrs}. The sets $A_i$ returned by \order() are the sets associated with each collapsed node, while a total order among them that respects the constraints implied by the edges in the DAG can be obtained either by depth-first search or by a topological sort \cite{clrs}. Hence, the \add()  and \order() operations have a worst case cost of $\Theta(n+m)$, where $m$ is the total number of edges in the graph. 

Several more efficient algorithms exist in literature (see, for example,\cite{tarjanfast,pearce,bender}), where the best (in terms of performance) proposed by \cite{bender} yielding a cost of $O(n)$ for \order() and an aggregate cost of at most $O(n^2\log n)$ for any sequence of \add\ operations. We stress here  that any of these more efficient implementations could be used for our purposes. We refer the reader interested in such implementations to  \cite{tarjanfast,pearce,bender} and, to avoid any ambiguity, we assume the above na\"ive approach for the remainder of this work.

\subsection{Greedy Content Search}
Our learning algorithm implements greedy content search, as described in Section~\ref{section:contentsearch}, in the following manner. When a new query is submitted to the database, the algorithm first selects a source $s$ uniformly at random. It then performs greedy content search 
using a memoryless selection policy $\hat{\myset{F}}$ with distribution $\hat{\ell}_{x}$, \emph{i.e.},
\begin{align}\Pr(\myset{F}(\myset{H}_k,x_k) = w)=\hat{\ell}_{x_k}(w)\quad w\in\Nn. \label{hatpolicy}\end{align}
Below, we discuss in detail how  $\hat{\ell}_{x}$, $x\in\Nn$, are computed. 

 When the current object $x_k$, $k=0,1,\ldots$, is equal to $x$, the algorithm evokes $\myset{O}_{x_k}$.\order() and obtains an ordered partition $A_1,A_2,\ldots,A_{j}$ of items in $\Nn\setminus\{x\}$. We define $$\hat{r}_x(w)=\sum_{j=1}^{i:w\in A_i} \hat{\mu}(A_j), \quad w\in \Nn\setminus\{x\}.$$
This can be seen as an ``estimator'' of the true rank $r_x$ given by \eqref{rank}. The distribution $\hat{\ell}_{x}$ is then computed as follows:
\begin{align}\hat{\ell}_x(w) = \frac{\hat{\mu}(w)}{\hat{r}_x(w)}\frac{1-\epsilon}{\hat{Z}_x}+ \frac{\epsilon}{n-1},\quad i=1,\ldots,n-1, \label{hatell}\end{align}
where $\hat{Z}_x=\sum_{w\in \Nn\setminus\{x\}} \hat{\mu}(w)/\hat{r}_x(w)$ is a normalization factor and $\epsilon>0$ is a small constant.
An alternative view of \eqref{hatell} is that the object proposed is selected uniformly at random with probability $\epsilon$, and proportionally to $\hat{\mu}(w_i)/\hat{r}_x(w_i)$ with probability $1-\epsilon$. The use of $\epsilon>0$ guarantees that every search eventually finds the target $t$. 

Upon locating a target $t$, any access to the oracle in the history $\history_k$ can be used to update  $\myset{O}_t$; in particular, a call \oracle($u$,$v$,$t$) that returns $u$ implies the constraint $u\lsteq{t} v$, which should be added to the data structure through $\myset{O}_t$.\add($u,v$). Note that this operation can take place only at the \emph{end of the greedy content search}; the outcomes of calls to the oracle can be observed, but the target $t$ is  revealed only after it has been located.

Our main result is that, as $\tau$ tends to infinity, the above algorithm achieves performance guarantees arbitrarily close to the ones of Theorem~\ref{th:content-search}.
Let $\hat{\myset{F}}(\tau)$ be the selection policy defined by \eqref{hatpolicy} at timeslot $\tau$ and denote by 
$$\bar{C}(\tau) = \sum_{(s,t)\in \Nn\times \Nn}\lambda(s,t)\sum_{s\in \Nn} \expect[C_{\hat{F}(\tau)}(s,t)] $$
the expected search cost at timeslot $\tau$. Then the following theorem holds:
\begin{theorem}\label{convergence}
Assume that for any two targets $u,v\in\Nn$, $\lambda(u,v)>0$.
$$\limsup_{\tau\to\infty}\bar{C}(\tau)\leq \frac{6 c^3(\mu) H(\mu) H_{\max}(\mu)}{(1-\epsilon)}$$
where $c(\mu)$, $H(\mu)$ and $H_{\max}(\mu)$ are the doubling parameter, the entropy and the max entropy, respectively, of the target distribution $\mu$.
\end{theorem}
The proof of this theorem can be found in Section~\ref{proof:convergence}.

\section{Analysis}\label{section:proofs}

This section includes the proofs of our theorems.

\subsection{Proof of Theorem~\protect\lowercase{\ref{hardnessthm}}}\label{proof:hardness}
We first prove that the randomized version of \textsc{SWND} is no harder than its deterministic version. 
Define $\textsc{DetSWND}$ to be the same as \textsc{SWND} with the additional restriction that $\myset{S}$ is deterministic. 
For any random variable $\mathcal{S}\subset \Nn$ that satisfies $|\mathcal{S}|\leq \beta$, there exists a deterministic set ${S}^*$ s.t. $|{S}^*|\leq \beta$ and
$\bar{C}_{{S}^*}\leq \bar{C}_{\myset{S}}$.
In particular, this is true for $$S^*=\arg\min_{S\in\Nn,|S|\leq \beta} C_{S}(s,t).$$
Thus, \textsc{SWND} 
is equivalent to \textsc{DetSWND}. 
In particular, any solution of \textsc{DetSWND} 
will also be a solution of  \textsc{SWND}. 
Moreover, given a solution $\mathcal{S}$ of \textsc{SWND} 
any deterministic $S$ belonging to the support of $\mathcal{S}$ will be a  solution of \textsc{DetSWND}.
We therefore turn our attention on \textsc{DetSWND}. Without loss of generality, we can assume that the weights $\lambda(s,t)$ are arbitrary non-negative numbers, as dividing every weight by $\sum_{s,t}\lambda(s,t)$ does not change the optimal solution. 
The decision problem  corresponding to  \textsc{DetSWND} is as follows
\begin{quote}
\textsc{DetSWND-D}: 
Given an embedding of $\Nn$ into $(\myset{M},d)$, a set of local edges $\myset{L}$,  a non-negative weight function $\lambda$, and two constants $\alpha>0$ and $\beta>0$,
 is there a directed edge set $S$ such that $|S|\leq\beta$ and $\sum_{(s,t)\times \Nn\times\Nn}\lambda(s,t)C_{{S}}(s,t)\leq\alpha?$
\end{quote}
Note that, given the set of shorcut edges $S$, forwarding a message with greedy forwarding from any $s$ to $t$ can take place in polynomial time. As a result, \textsc{DetSWND-D} is in NP. 
We will prove it is also NP-hard by reducing the following NP-complete problem to it:
\begin{quote}
\textsc{DominatingSet}: Given a graph $G(V,E)$ and a constant $k$, is there a set $A\subseteq V$ such that $|A|\leq k$ and $\Gamma(A)\cup A=V$, where $\Gamma(A)$ the neighborhood of $A$ in $G$?
\end{quote}

Given an instance $(G(V,E),k)$ of \textsc{DominatingSet}, we construct an instance of  \textsc{DetSWND-D} as follows.
The set $\Nn$ in this instance will be embedded in a 2-dimensional grid, and the distance metric $d$ will be the Manhattan distance on the grid. In particular, let $n=|V|$ be the size of the graph $G$ and, w.l.o.g., assume that $V=\{1,2,\ldots,n\}$. Let 
\begin{align}
\ell_0&= 6n+3, \\
\ell_1&= n\ell_0+2= 6n^2+3n+2,\label{l1}\\
\ell_2&= \ell_1+3n +1= 6n^2+6n+3.\label{l2}\\
\ell_3&=\ell_0= 6n+3,\label{l0l3}
\end{align}
 We construct a $n_1\times n_2$ grid, where $n_1 = (n-1)\cdot \ell_0 +1$ and $n_2=\ell_1+\ell_2+\ell_3+1$. That is, the total number of nodes in the grid is $$N = [(n-1)\cdot \ell_0 +1]\cdot (\ell_1+\ell_2+\ell_3+1)=\Theta(n^4).$$
The object set $\Nn$ will be the set of nodes in the above grid, and the metric space will be $(\mathbbm{Z}^2,d)$ where $d$ is the Manhattan distance on $\mathbbm{Z}^2$. The local edges $\myset{L}$ is defined according to \eqref{local} with $r=1$, \emph{i.e.},
and any two adjacent nodes in the grid are connected by an edge in $\myset{L}$.

Denote by $a_i$, $i=1,\ldots,n$, the node on the first column of the grid that resides at row $(i-1)\ell_0+1$. Similarly, denote by $b_i$, $c_i$ and $d_i$  the nodes on the columns $(\ell_1+1)$, $(\ell_1+\ell_2+1)$ and $(\ell_1+\ell_2+\ell_3+1)$  the grid, respectively, that reside at the same row as $a_i$, $i=1,\ldots,n$. These nodes are depicted in Figure~\ref{graph}.
We define the weight  function $\lambda(i,j)$ over the pairs of nodes in the grid as follows. The pairs of grid nodes that receive a non-zero weight are the ones belonging to one of the following sets:
\begin{align*}
A_1 &= \{(a_i,b_i)\mid i\in V\},\\ 
A_2 &= \{(b_i,b_j)\mid (i,j)\in E \}\!\cup\! \{(c_i,d_j)\mid (i,j)\in E\}\!\cup\!\{(c_i,d_i)\mid i\in V\},\\
A_3 &= \{(a_i,d_i)\mid i\in V\}.
\end{align*} 
The sets $A_1$ and $A_2$ are depicted in Fig.~\ref{graph} with dashed and solid lines, respectively. Note that 
$|A_1|= n$ as it contains one pair for each vertex in $V$, $|A_2|=4|E|+n$ as it contains four pairs for each edge in $E$ and one pair for each vertex in $V$, and, finally, $|A_3| = n.$
The pairs in $A_1$ receive a weight equal to 
$W_1 = 1,$ the pairs in $A_2$ receive a weight equal to $W_2 = 3n+1$ and the pairs in $A_3$ receive a weight equal to
$ W_3=1.$
\begin{figure}[t]
\begin{center}
\input{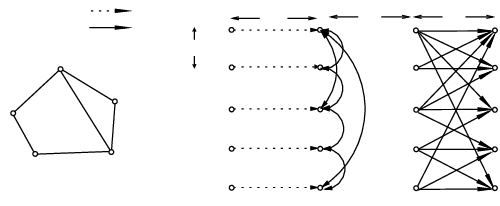}
\end{center}
\caption{\small A reduction of an instance of \textsc{DominatingSet} to an instance of \textsc{DetSWND-D}. Only the nodes on the grid that have non-zero incoming or outgoing demands (weights) are depicted. The dashed arrows depict $A_1$, the set of pairs that receive a weight $W_1$. The solid arrows depict $A_2$, the set of pairs that receive weight $W_2$.
}\label{graph}
\end{figure}

For the bounds $\alpha$ and $\beta$ take
\begin{eqnarray}
\alpha &\!=\! & 2W_1|\!A_1\!|\! +\! W_2|\!A_2\!|\! +\!3|\!A_3\!|W_3\nonumber\\ 
&\!=\! & (3n\!+\!1)(4|\!E|\!+\!n)\!+\!5n\label{alpha}
\end{eqnarray}
and
\begin{eqnarray}
\beta &=& |A_2|+n+k \nonumber \\
&=& 4|E|+2n+k \label{beta}.
\end{eqnarray}

The above construction can take place in polynomial time in $n$. Moreover, if the graph $G$ has a dominating set of size no more than $k$, one can construct a deterministic set of shortcut edges $\myset{S}$ that satisfies the constraints of \textsc{DetSWND-D}.

\begin{lemma} If the instance of \textsc{DominatingSet} is a ``yes'' instance, then the constructed instance of \textsc{DetSWND-D} is also a ``yes'' instance. 
\end{lemma}

\begin{proof}
To see this, suppose that there exists a dominating set $A$ of the graph with size $|A|\leq k$. Then, for every $i\in V\setminus A$, there exists a $j\in A$ such that $i\in \Gamma(j)$, \emph{i.e.}, $i$ is a neighbor of $j$. We construct $\myset{S}$ as follows.
 For every $i \in A$, add the  edges $(a_i,b_i)$ and $(b_i,c_i)$  in $\myset{S}$.
For every $i \in V\setminus A$, add an edge $(a_i,b_j)$ in $\myset{S}$, where $j$ is such that $j\in A$ and  $i\in\Gamma(j)$.
 For every pair in $A_2$, add this edge in $\myset{S}$.
\begin{figure}[t]
\begin{center}
\input{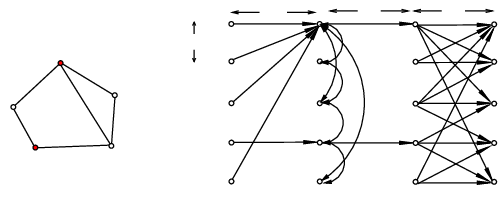}
\end{center}
\caption{\small A ``yes'' instance of \textsc{DominatingSet} and the corresponding ``yes'' instance of \textsc{DetSWND-D}. The graph on the left is can be dominated by two nodes, 1 and 4. The corresponding set $\myset{S}$ of shortcut contacts that satisfies the constraints of  \textsc{DetSWND-D} is depicted on the right.
}\label{instance}
\end{figure}
The size of $\myset{S}$ is 
\begin{eqnarray*}
|\myset{S}| &=& 2|A|+(|V|-|A|)+ |A_2|\\ &=& |A|+n+4|E|+n\\&\leq& 4|E|+2n+k.
\end{eqnarray*}
Moreover, the weighted forwarding distance is
\begin{align*}
\bar{C}_{\mathcal{S}}^w& =\!\! \sum_{(i,j)\in A_1}\!\!\!\!W_1 C_{\mathcal{S}}(i,j)\! +\!\!\sum_{(i,j)\in A_2}\!\!\!\!W_2 C_{\mathcal{S}}(i,j)\! +\!
\!\sum_{(i,j)\in A_3}\!\!\!\!W_3 C_{\mathcal{S}}(i,j).
\end{align*}
We have
$$\sum_{(i,j)\in A_2}W_2 C_{\mathcal{S}}(i,j) = W_2 |A_2|$$ 
as every pair in $A_2$ is connected by an edge in $\myset{S}$.
Consider now a pair  $a_i,b_i)\in A_1$, $i\in V$. There is exactly one edge in $\myset{S}$ departing from $a_i$ which has the form $(a_i,b_j)$, where where either $j=i$ is or $j$ a neighbor of $i$. The distance of the closest local neighbor of $a_i$ from $b_i$ is $\ell_1-1$. The distance of $b_j$ from $b_i$ is at most $n\cdot \ell_0$. As
$\ell_1-1=n\ell_0+2-1>n\ell_0 $ 
greedy forwarding will follow  $(a_i,b_j)$. If $b_j=b_i$, then $C_{\mathcal{S}}(a_i,b_i)=1$. If $b_j\neq b_i$, as $j$ is a neighbor of $i$, $\myset{S}$ contains the edge $(b_j,b_i)$. Hence, if $b_j\neq b_i$, $C_{\mathcal{S}}(a_i,b_i)=2$. As $i$ was arbitrary, we get that 
$$\sum_{(i,j)\in A_1}W_1 C_{\mathcal{S}}(i,j)\leq 2W_1n.$$ 

Next, consider a pair $(a_i,d_i)\in A_3$. For the same reasons as for the pair $(a_i,b_i)$, the shortcut edge $(a_i,b_j)$ in $\myset{S}$ will be used by the greedy forwarding  algorithm. In particular, the distance of the closest local  neighbor of $a_i$ from $d_i$ is $\ell_1+\ell_2+\ell_3-1$ and $d(b_j,d_i)$ is at most $\ell_2+\ell_3+n\cdot\ell_0$. 
 As $\ell_1-1>n\ell_0 $,
greedy forwarding will follow  $(a_i,b_j)$. 

By the construction of $\myset{S}$, $b_j$ is such that $j\in A$. As a result, again by the construction of $\myset{S}$, $(b_j,c_j)\in \myset{S}$. The closest local neighbor of $b_j$ to $d_i$ has  $\ell_2+\ell_3+d(b_j,b_i)-1$ Manhattan distance from $d_j$. Any shortcut neighbor $b_k$ of $b_j$ has at least $\ell_2+\ell_3$ Manhattan distance from $b_i$. On the other hand, $c_j$ has $\ell_3+d(b_j,b_i)$ Manhattan distance from $d_i$. As 
$\ell_2>1$ and $\ell_2>n\ell_0\geq d(b_j,b_i)$, the greedy forwarding algorithm will follow $(b_j,c_j)$. Finally, as $A_2\subset \myset{S}$, and $j=i$ or $j$ is a neighbor of $i$, the edge $(c_j,d_i)$ will be in $\myset{S}$. Hence, the greedy forwarding algorithm will reach $d_j$ in exactly 3 steps. As $i\in V$ was arbitrary, we get that
$$\sum_{(i,j)\in A_3}W_3 C_{\mathcal{S}}(i,j)= 3W_3n.$$ 
Hence, 
$$\bar{C}_{\mathcal{S}}^w \leq 2W_1n+ W_2 |A_2|+ 3W_3n = \alpha$$
and, therefore, the instance of \textsc{DetSWND-D} is a ``yes'' instance.
\end{proof}

To complete the proof, we show that a dominating set of size $k$ exists only if there exists a $\myset{S}$ that satisfies the constraints in constucted  instance of \textsc{DetSWND-D}.

\begin{lemma}\label{dettover}
 If the constucted  instance of \textsc{DetSWND-D} is a ``yes'' instance, then the  instance of \textsc{DominatingSet} is also a ``yes'' instance. 
\end{lemma}
\begin{proof}
\sloppy Assume that 
 there exists a set $\myset{S}$, with $|\myset{S}|\leq \beta$  such that the augmented graph has a weighted forwarding distance less than or equal to $\alpha$.
\fussy
Then \begin{align}A_2\subseteq \myset{S}\label{lemmaS2}.\end{align} 
To see this, suppose that $A_2\not\subseteq \myset{S}$.  Then, there is at least one pair of nodes $(i,j)$ in $A_2$ with $C_{\mathcal{S}}(i,j)\geq 2$. Therefore, 
\begin{align*}
\bar{C}_{\myset{S}}^w& \geq 1\cdot W_1|A_1|+[(|A_2|-1)\cdot 1+  2]\cdot W_2+1\cdot W_3 |A_3|\\
& =
 (3n+1)(4|E|+n)+5n+1{>}\alpha,
\end{align*}
a contradiction.

Essentially, by choosing $W_2$ to be large, we enforce that all ``demands'' in $A_2$ are satisfied by a direct edge in $\myset{S}$. The next lemma shows a similar result for $A_1$. Using shortcut edges to satisfy these ``demands'' is enforced by making the distance $\ell_1$ very large.
\begin{lemma}\label{lemmaS1}
For every $i\in V$, there exists at least one shortcut edge in $\myset{S}$ whose origin is in the same row as $a_i$ and in a column to the left of $b_i$. Moreover, this edge is used during the greedy forwarding of a message from $a_i$ to $b_i$.
\end{lemma}
\begin{proof}
Suppose not. Then, there exists an $i\in V$ such that no shortcut edge has its origin  between $a_i$ and $b_i$, or such an edge exists but is not used by the greedy forwarding from $a_i$ to $b_i$ (\emph{e.g.}, because it points too far from $b_i$). Then, the greedy forwarding from $a_i$ to $b_i$ will use only local edges and, hence, $C_{\mathcal{S}}(a_i,b_i)=\ell_1.$
We thus have that
\begin{eqnarray*}
\bar{C}^w_\myset{S} &\geq& 
\ell_1+2n-1+W_2|A_2|\\ & \stackrel{\eqref{l1}}{=}&6n^2+5n+1+W_2|A_2|
\end{eqnarray*}
%
On the other hand, by \eqref{alpha}
$\alpha = 5n+ W_2|A_2|$
so $\bar{C}^w_\myset{S}>\alpha$, a contradiction.
\end{proof}
Let $S_1$ be the set of all edges whose origin is between some $a_i$ and $b_i$, $i\in V$, and that are used during forwarding from this $a_i$ to $b_i$. Note that Lemma~\ref{lemmaS1} implies that $|S_1|\geq n$.
 The target of any edge in $S_1$ must lie to the left of the $2\ell_1+1$-th column of the grid
This is because the Manhattan distance of $a_i$ to $b_i$ is $\ell_1$, so its left local neighbor lies at $\ell_1-1$ steps from $b_i$. Greedy forwarding is monotone, so the Manhattan distance from $b_i$ of any target of an edge followed subsequently to route towards $b_i$ must  be less than $\ell_1$. 

Essentially, all edges in $S_1$ must point close enough to $b_i$, otherwise they would not be used in greedy forwarding. This implies that, to forward the ``demands'' in $A_3$ 
an \emph{additional} set of shortcut edges 
need to be used.
\begin{lemma}\label{lemmaL3} For every $i\in V$, there exists at least one shortcut edge in $\myset{S}$ that is used when forwarding a message from $a_i$ to $d_i$ that is neither in $S_1$ nor in $A_2$.
\end{lemma}
\begin{proof}
Suppose not. We established above that the target of any edge in $S_1$ is to the left of the $2\ell_1+1$ column. Recall that
$ A_2 = \{(b_i,b_j)\mid (i,j)\in E \}\cup \{(c_i,d_j)\mid (i,j)\in E\}\cup\{(c_i,d_i)\mid i\in V\}.$
By the definition of $b_i$, $i \in V$, the targets of the edges in $\{(b_i,b_j)\mid (i,j)\in E \}$ lie on the $(\ell_1+1)$-th column. Similarly, the origins of the edges in $\{(c_i,d_j)\mid (i,j)\in E\}\cup\{(c_i,d_i)\mid i\in V\}$ lie on the $\ell_1+\ell_2+1$-th column.
As a result, if the lemma does not hold, there is a demand in $A_3$, say $(a_i,d_i)$, that does not use any additional shortcut edges. This means that the distance between the $2\ell+1$ and the  $\ell_1+\ell_2+1$-th column is traversed by using local edges. Hence, 
$C_{\mathcal{S}}(a_i,d_i)\geq \ell_2-\ell_1+1 $
as at least one additional step is needed to get to the $2\ell_1+1$-th column from $a_i$.
This implies that 
\begin{eqnarray*}
\bar{C}_{\mathcal{S}}^w &\geq& 
=2n+W_2|A_2|+\ell_2-\ell_1\\
&\stackrel{\eqref{l2}}{=}&  W_2|A_2|+5n+1>\alpha,
\end{eqnarray*}
a contradiction.
\end{proof}
Let $S_3=\myset{S}\setminus(S_1\cup A_2)$. Lemma~\ref{lemmaL3} implies that $S_3$ is non-empty, while \eqref{lemmaS2} and Lemma \ref{lemmaS1}, along with the fact that $|\myset{S}|\leq\beta= |A_2|+n+k$, imply that $|S_3|\leq k$. The following lemma states that some of these edges must have targets that are close enough to the destinations $d_i$.
\begin{lemma}\label{lemmafromdest} For each $i\in V$, there exists an edge in $S_3$ whose target is within Manhattan distance $3n+1$ of either $d_i$ or $c_j$, where $(c_j,d_i)\in A_2$. Moreover, this edge is used for forwarding a message from $a_i$ to $d_i$ with greedy forwarding.
\end{lemma}
\begin{proof}
Suppose not. Then there exists an $i$ for which greedy forwarding from $a_i$ to $d_i$ does not employ any edge fitting the description in the lemma. 
Then, the destination $d_i$  can not be reached by a shortcut edge in either $S_3$ or $A_1$ whose target is closer than $3n+1$ steps. Thus, $d_i$ is reached in one of the two following ways: either  $3n+1$ steps are required in reaching it, through forwarding over local edges, or an edge $(c_j,d_i)$ in $A_2$ is used to reach it. In the latter case,  reaching $c_i$ also requires at least $3n+1$ steps of local forwarding, as no edge in $A_2$ or $S_3$ has an target within $3n$ steps from it, and any edge in $S_1$ that may be this close is not used (by the hypothesis).
As a result,
$C_{\mathcal{S}}(a_i,d_i)\geq 3n+2 $
as at least one additional step is required in reaching the ball of radius $3n$ centered around $d_i$ or $c_i$ from $a_i$.
This gives 
$$\bar{C}_{\mathcal{S}}^w \geq
 5n + W_2|A_2|+1> \alpha,$$
a contradiction.
\end{proof}
When forwarding from $a_i$ to $d_i$, $i \in V$, there may be more than one edges in $S_3$ fitting the description in Lemma~\ref{lemmafromdest}. For each $i\in V$, consider the \emph{last} of all these edges. Denote the resulting subset by $S_3'$. By definition, $|S_3'|\leq |S_3|\leq k$. 
For each $i$, there exists \emph{exactly one} edge in $S_3'$ that is used to forward a message from $a_i$ to $d_i$. 
Moreover, recall that $\ell_0=\ell_3=6n+3$. Therefore, the Manhattan distance between any two nodes in $\{c_1,\ldots,c_n\}\cup\{d_1,\ldots,d_n\}$ is $2 (3n+1)+1$. As a result, the targets of the edges in $S_3'$ will be within distance $3n+1$ of \emph{exactly one} of the nodes in the above set. 

Let $A\subset V$ be the set of all vertices $i\in V$ such that the unique edge in $S_3'$ used in forwarding from $a_i$ to $d_i$ has an target within distance $3n+1$ of either $c_i$ or $d_i$. Then $A$ is a dominating set of $G$, and $|A|\leq k$. To see this, note first that  $|A|\leq k$ because each target of an edge in $S_3'$ can be within distance $3n+1$ of only one of the nodes in $\{c_1,\ldots,c_n\}\cup\{d_1,\ldots,d_n\}$, and there are at most $k$ edges in $S_3'$.

To see that $A$ dominates the graph $G$, suppose that $j\in V\setminus A$. Then, by Lemma~\ref{lemmafromdest}, the edge in $S_3'$ corresponding to $i$ is either pointing within distance $3n+1$ of either $d_j$ or a $c_i$ such that $(c_i,d_j)\in A_2$. By the construction of $A$, it cannot point in the proximity of $d_j$, because then $j\in A$, a contradiction. Similarly, it cannot point in the proximity of $c_j$, because then, again, $j\in A$, a contradiction. Therefore, it points in the proximity of some $c_i$, where $i\neq j$ and $(c_i,d_j)\in A_2$. By the construction of $A$, $i\in A$. Moreover, by the definition of $A_2$, $(c_i,d_j)\in A_2$ if and only if $(i,j)\in E$. Therefore, $j\in \Gamma(A)$. As $j$ was arbitrary, $A$ is a dominating set of $G$.  
\end{proof}

\subsection{Proof of Theorem~\protect\lowercase{\ref{th:small-world}}}\label{proof:small-world}

According to \eqref{shortcutdistr},  the probability that object $x$ links to $y$ is given by $\ell_x(y)=\frac{1}{Z_x} \frac{\mu(y)}{r_x(y)},$
where 
$$Z_x=\sum_{y\in \myset{T}}\frac{\mu(y)}{r_x(y)}$$
is a normalization factor bounded as follows.
\begin{lemma}\label{normalization}
For any $x\in \Nn$, let $x^*\in\min_{\lsteq{x}} \myset{T}$  be any object in $\myset{T}$ among the closest targets to $x$. Then
$$Z_x\leq 1+\ln( 1/{\mu(x^*)})\leq 3H_{\max}. $$
\end{lemma}
\begin{proof}
Sort the target set $\myset{T}$ from the closest to furthest object from $x$ and index objects in an increasing sequence $i=1,\ldots,k$, so the objects at the same distance from $x$ receive the same index. Let $A_i$, $i=1,\ldots,k$, be the set containing objects indexed by $i$, and let  $\mu_i=\mu(A_i)$ and $\mu_0=\mu(x)$. Furthermore, let $Q_i=\sum_{j=0}^{i}\mu_j$. Then 
$Z_x=\sum_{i=1}^{k}\frac{\mu_i}{Q_i}.$

Define $f_x(r):\mathbb{R^+}\rightarrow\mathbb{R}$ as $$f_x(r)=\frac{1}{r}-\mu(x).$$ Clearly, $f_x(\frac{1}{Q_i})=\sum_{j=1}^i \mu_j$, for $i\in\{1,2\dots,k\}$. This means that we can rewrite $Z_x$ as
$$Z_x=\sum_{i=1}^{k}(f_x(1/{Q_i})-f_x(1/{Q_{i-1}}))/{Q_i}.$$
By reordering the terms involved in the sum above, we get
$$Z_x=  {f_x(\frac{1}{Q_{k}})}/{Q_{k}}+\sum_{i=1}^{k-1}f_x(1/Q_i)\left(\frac{1}{Q_i}-\frac{1}{Q_{i+1}}\right).$$
First note that $Q_{k}=1$, and second that since $f_x(r)$ is a decreasing function, 
\begin{eqnarray*}
Z_x &\leq& 1-\mu_0+\int_{{1}/{Q_{k}}}^{{1}/{Q_1}} f_x(r) dr\\
&=& 1-\frac{\mu_0}{Q_1}+\ln\frac{1}{Q_1}.
\end{eqnarray*}

This shows that if $\mu_0=0$ then $Z_x\leq 1+\ln\frac{1}{\mu_1}$  or otherwise $Z_x\leq 1+\ln\frac{1}{\mu_0}$ .
\end{proof}
%

Given the set $\Ss$, recall that $C_{\Ss}(s,t)$ is the number of steps required by the greedy forwarding to reach $t\in \Nn$ from $s\in \Nn$. 
 We say that a message at object $v$ is in \textit{phase} $j$ if $$2^j \mu(t)\leq r_t(v)\leq2^{j+1}\mu(t).$$ 
Notice that   the number of different phases  is at most $\log_2 1/\mu(t)$. We can write $C_{\Ss}(s,t)$ as 
\begin{equation}\label{phases}
C_{\Ss}(s,t)=X_1+X_2+\cdots +X_{\log \frac{1}{\mu(t)}},
\end{equation}
where $X_j$ are the hops occurring in phase $j$.
Assume that $j>1$, and let 
$$I = \left\{w\in \Nn: r_t(w)\leq \frac{r_t(v)}{2}\right\}.$$ The probability that $v$ links to an object in the set $I$, and hence moving to phase $j-1$, is 
$$\sum_{w\in I}\ell_{v,w}=\frac{1}{Z_v}\sum_{w\in I} \frac{\mu(w)}{r_v(w)}.$$
Let $\mu_t(r)=\mu(B_t(r))$  and $\rho>0$ be the smallest radius such that $\mu_t(\rho)\geq r_t(v)/2$. Since we assumed that $j>1$ such a $\rho>0$ exists. Clearly, for any $r<\rho$ we have $\mu_t(r)< r_t(v)/2$. In particular, 
\begin{equation}\label{right}
\mu_t(\rho/2)< \frac{1}{2}r_t(v). 
\end{equation}

On the other hand, since the doubling parameter is $c(\mu)$ we have 
\begin{equation}\label{left}
\mu_t(\rho/2)> \frac{1}{c(\mu)}\mu_t(\rho)\geq \frac{1}{2c(\mu)}r_t(v). 
\end{equation}

Therefore, by combining \eqref{right} and \eqref{left} we obtain
\begin{equation}\label{left-right}
\frac{1}{2c(\mu)}r_t(v)<\mu_t(\rho/2)<\frac{1}{2}r_t(v).
\end{equation}
Let $I_\rho=B_t(\rho)$ be the set of objects within radius $\rho/2$ from $t$.
Then $I_\rho\subset I$, so 
$$\sum_{w\in I}\ell_{v,w}\geq  \frac{1}{Z_v}\sum_{w\in I_\rho} \frac{\mu(w)}{r_v(w)}.$$
By triangle inequality, for any $w\in I_\rho$ and $y$ such that $d(y,v)\leq d(v,w$) we have
\begin{eqnarray*}
d(t,y)&\overset{(a)}{\leq}& d(v,y)+d(v,t)\\
&\leq & d(w,y)+d(v,t)\\
&\overset{(b)}{\leq}& d(t,w)+d(v,t)+d(v,t)\\
&\overset{(c)}{\leq}& \frac{1}{2}d(v,t)+d(v,t)+d(v,t)\\
&\leq & \frac{5}{2}d(v,t),
\end{eqnarray*}
where in $(a)$ and $(b)$ we used the triangle inequality and in $(c)$ we used the fact that $\rho/2<d(v,t)/2$.
This means that $r_v(w)\leq \mu_t(\frac{5}{2}d(v,t)),$ and consequently, $r_v(w)\leq  c^2(\mu) r_t(v). $ Therefore, 
\begin{eqnarray*}
 \sum_{w\in I}\ell_{v,w}& \geq& \frac{1}{Z_v} \frac{\sum_{w\in I_\rho}\mu(w)}{c^2(\mu) r_t(v)}\\ 
 &=& \frac{1}{Z_v} \frac{\mu_t(\rho/2)}{c^2(\mu) r_t(v)}. 
\end{eqnarray*}
By \eqref{left-right}, the probability of terminating phase $j$ is uniformly bounded by
\begin{eqnarray}
\sum_{w\in I}\ell_{v,w} &\geq& \min_{v}\frac{1}{2c^3(\mu) Z_v}\nonumber\\ &\overset{\text{Lem.}~\ref{normalization}}{\geq}& \frac{1}{6c^3(\mu) H_{\max}(\mu)} \label{prob-phase}
\end{eqnarray}
 As a result, the probability of terminating phase $j$ is stochastically dominated by a geometric random variable with the parameter given in \eqref{prob-phase}. This is because (a) if the current object does not have a shortcut edge which lies in the set $I$, by Property~\ref{greedyprop}, greedy forwarding sends the message to one of the neighbours that is closer to $t$ and (b) shortcut edges are sampled independently across neighbours. Hence, given that $t$ is the target object and $s$ is the source object, 
\begin{equation}\label{cost-conditional}
\E[X_j|s,t]\leq 6c^3(\mu)H_{\max}(\mu).
\end{equation}
Suppose now that $j=1$. By the triangle inequality, $B_v(d(v,t))\subseteq B_t(2d(v,t))$ and $r_v(t)\leq c(\mu) r_t(v)$. Hence, 
$$
\ell_{v,t}\geq \frac{1}{Z_v}\frac{\mu(t)}{c(\mu) r_t(v)}\overset{}{\geq} \frac{1}{2c(\mu)Z_v}\geq\frac{1}{6c(\mu)H_{\max}(\mu)}$$
since object $v$ is in the first phase and thus $\mu(t)\leq r_t(v)\leq 2\mu(t)$. Consequently,
\begin{equation}\label{cost-first}
\E[X_1|s,t]\leq 6c(\mu)H_{\max}(\mu).
\end{equation}
Combining \eqref{phases}, \eqref{cost-conditional}, \eqref{cost-first} and using the linearity of expectation, we get
$$\E[C_{\Ss}(s,t)]\leq 6 c^3(\mu) H_{\max}(\mu) \log \frac{1}{\mu(t)}$$ and, thus,
$\bar{C}_{\myset{S}}
\leq 6 c^3(\mu) H_{\max}(\mu) H(\mu).$

\subsection{Proof of Theorem\protect\lowercase{~\ref{th:content-search}}}\label{proof:content-search}
The idea of the proof is very similar to the previous one and follows the same path. Recall that the  selection policy is memoryless and determined by $$\prob(\myset{F}(\history_k,x_k)=w)=\ell_{x_k}(w).$$
We assume that the desired object is $t$ and the content search starts from $s$. Since there are no local edges, the only way that the greedy search moves from the current object $x_k$ is by proposing an object that is closer to $t$. 
 Like in the SWND case, we are in particular interested in bounding the probability  that the rank of the proposed object is roughly half the rank of the  current object. This way we can compute how fast we make progress in our search.  

As the search moves from $s$  to $t$ we say that the search is in phase  $j$ when the rank of the current object $x_k$ is between $2^j\mu(t)$ and $2^{j+1}\mu(t)$. As stated earlier, the greedy search algorithm keeps making comparisons until it finds another object closer to $t$.  We can write  $C_{\Ff}(s,t)$ as $$C_{\myset{F}}(s,t)=X_1+X_2+\cdots +X_{\log \frac{1}{\mu(t)}},$$ where $X_j$ denotes the number of comparisons done by comparison oracle in phase $j$. Let us consider a particular phase $j$ and denote $I$  the set of objects whose ranks from $t$ are at most $r_t(x_k)/2$. Note that phase $j$ will terminate if  the comparison oracle proposes an object from  set $I$. The probability that this happens is $$ \sum_{w\in I} \prob(\myset{F}(\history_k,x_k)=w)=\sum_{w\in I}\ell_{x_k,w} .$$ Note that the sum on the right hand side  depends on the distribution of shortcut edges and is independent of local edges. To bound this sum we can use \eqref{prob-phase}. Hence, with probability at least $1/(6c^3(\mu)H_{\max}(\mu))$,  phase $j$ will terminate. In other words, using the above selection policy, if the current object $x_k$ is in phase $j$, with probability  $1/(6c^3(\mu)H_{\max}(\mu))$ the proposed object will be in phase $(j-1)$. This defines a geometric random variable which yields to the fact that on average the number of queries needed to halve the rank is at most $6c(\mu)^3H_{\max}$ or $\E[X_j|s,t]\leq 6c(\mu)^3H_{\max}.$ Taking average over the demand $\lambda$, we can conclude that the average number of comparisons is less than $\bar{C}_{\myset{F}}\leq 6c^3(\mu)H_{\max}(\mu) H(\mu).$ 

\subsection{Proof of Theorem~\protect\lowercase{\ref{lowerbound}}}\label{proofoflowerbound}
Our proof amounts to constructing a metric space and a target distribution $\mu$ for which the bound holds. 
Our construction will be as follows. For some integers $D,K$, the target set ${\mathcal{N}}$ is taken as ${\mathcal{N}}=\{1,\ldots ,D\}^K$. The distance $d(x,y)$ between two distinct elements $x,y$ of ${\mathcal{N}}$ is defined as $d(x,y)=2^{m}$, where
$$
m=\max\left\{i\in\{1,\ldots,K\}: x(K-i)\ne y(K-i)\right\}.
$$
We then have the following
\begin{lemma}
\label{prop1} 
Let $\mu$ be the uniform distribution over $\myset{N}$. Then (i)  $c(\mu)=D$, and
(ii) if the target distribution is $\mu$, the optimal average search cost $C^*$ based on a comparison oracle satisfies
$C^*\ge K \frac{D-1}{2}$.
\end{lemma}
Before proving Lemma~\ref{prop1}, we note that Theorem~\ref{lowerbound} immediately follows as a corollary.
\begin{proof}
Part (i): Let $x=(x(1),\ldots x(K))\in{\mathcal{N}}$, and fix $r>0$. 
Assume first that $r<2$; then, the ball $B(x,r)$ contains only $x$, while the ball $B(x,2r)$ contains either only $x$ if $r<1$, or precisely those $y\in{\mathcal{N}}$ such that $$(y(1),\ldots,y(K-1))=(x(1),\ldots,x(K-1))$$ if $r\ge 1$. In the latter case $B(x,2r)$ contains precisely $D$ elements. Hence, for such $r<2$, and for the uniform measure on ${\mathcal{N}}$, the inequality 
\begin{equation}
\label{ddd}
\mu(B(x,2r))\le D \mu(B(x,r))
\end{equation}
holds, and with equality if in addition $r\ge 1$.

Consider now the case where $r\ge 2$. 
Let the integer $m\ge 1$ be such that $r\in[2^{m},2^{m+1})$. By definition of the metric $d$ on ${\mathcal N}$, the ball $B(x,r)$  consists of all $y\in {\mathcal{N}}$ such that $$(y(1),\ldots,y(K-m))=(x(1),\ldots,x(K-m)),$$ and hence contains $D^{\min(K,m)}$ points. Similarly, the ball $B(x,2r)$ contains $D^{\min(K,m+1)}$ points. Hence (\ref{ddd}) also holds when $r\ge 2$. 

Part (ii): We assume that the comparison oracle, in addition to returning one of the two proposals that is closer to the target, also reveals the distance of the proposal it returns to the target. We further assume that upon selection of the initial search candidate $x_0$, its distance to the unknown target is also revealed. We now establish that the lower bound on $C^*$ holds when this additional information is available; it holds a fortiori for our more resticted comparison oracle.

We decompose the search procedure into phases, depending on  the current distance to the destination. Let $L_0$ be the integer such that the initial proposal $x_0$ is at distance $2^{L_0}$ of the target $t$, i.e. 
\begin{eqnarray*}
(x_0(1),\ldots,x_0(K-L_0))&=&(t(1),\ldots,t(K-L_0)),\\ x_0(K-L_0+1)&\ne & t(K-L_0+1).
\end{eqnarray*}

No information on $t$ can be obtained by submitting proposals $x$ such that $d(x,x_0)\ne 2^{L_0}$. Thus, to be useful, the next proposal $x$ must share its $(K-L_0)$ first components with $x_0$, and differ from $x_0$ in its $(K-L_0+1)$-th entry. Now, keeping track of previous proposals made for which the distance to $t$ remained equal to $2^{L_0}$, the best choice for the next proposal consists in picking it again at distance $2^{L_0}$ from $x_0$, but choosing for its $(K-L_0+1)$-th entry one that has not been proposed so far. It is easy to see that, with this strategy, the number of additional proposals after $x_0$ needed to leave this phase is uniformly distributed on $\{1,\ldots D-1\}$, the number of options for the $(K-L_0+1)$-th entry of the target.

A similar argument entails that the number of proposals made in each phase equals 1 plus a uniform random variable on $\{1,\ldots,D-1\}$. It remains to control the number of phases. We argue that it admits a Binomial distribution, with parameters $(K, (D-1)/D)$. Indeed, as we make a proposal which takes us into a new phase, no information is available on the next entries of the target, and for each such entry, the new proposal makes a correct guess with probability $1/D$. This yields the announced Binomial distribution for the numbers of phases (when it equals 0, the initial proposal $x_0$ coincided with the target).

Thus the optimal number of search steps $C$ verifies $C\ge \sum_{i=1}^X (1+Y_i),$
where the $Y_i$ are i.i.d., uniformly distributed on $\{1,\ldots,D-1\}$, and independent of the random variable $X$, which admits a Binomial distribution with parameters $(K, (D-1)/D)$. Thus using Wald's identity, we obtain that
$
\expect[C]\ge \expect[X]\expect[Y_1],
$
which readily implies (ii).
\end{proof}
Note that the lower bound in (ii) has been established for search strategies that utilize the entire search history.  Hence, it is \emph{not} restricted to memoryless search.

\subsection{Proof of Theorem~\protect\lowercase{\ref{convergence}}}\label{proof:convergence}

Let $\Delta_\mu= \sup_{x\in \Nn} |\hat{\mu}(x)-\mu(x)|$.
Observe first that, by the weak law of large numbers, for any $\delta>0$
\begin{align}\label{weak}\lim_{\tau\to\infty} \Pr(\Delta_\mu>\delta)=0.\end{align}
\emph{i.e.}, $\hat{\mu}$ converges to $\mu$ in probability.
The lemma below states, for every $t\in \Nn$, the  order data structure $\myset{O}_t$ will learn the correct order of any two objects $u,v$ in finite time. 
\begin{lemma}\label{learnorder}
Consider $u,v,t\in\myset{N}$ such that $u\lsteq{t} v$. Then, the order data structure in $t$ evokes $\myset{O}_t$.\emph{\add}($u$,$v$) after a finite time, with probability one.
\end{lemma}
\begin{proof}
Recall that $\myset{O}_t$.\emph{\add}($u$,$v$) is evoked if and only if a call $\oracle(u,v,t)$ takes place and it returns $u$. If $u\lst{t} v$ then $\oracle(u,v,t)=u$. If, on the other hand, $u\eqvl{t}v$, then $\oracle(u,v,t)$ returns $u$ with non-zero probability. It thus suffices to show that such, for large enough $\tau$, a call $\oracle(u,v,t)$ occurs at timeslot $\tau$ with a non-zero probability. 
By the  hypothesis of Theorem~\ref{convergence}, $\lambda(u,t)>0$. By \eqref{hatell}, given that the source is $u$, the probability that $\hat{\myset{F}}(u)=v$ conditioned on $\hat{\mu}$  is
\begin{align*}\hat{\ell}_u(v)& \geq \frac{\mu(v)\!-\!\Delta_\mu}{1\!+\!(n\!-\!1)\Delta_\mu}\frac{1\!-\!\epsilon}{n\!-\!1}\!+\!\frac{\epsilon}{n\!-\!1}\geq \frac{\mu(v)\!-\!\Delta_\mu}{(1\!+\!(n\!-\!1)\Delta_\mu)(n\!-\!1)}\end{align*}
as $\hat{Z}_v\leq n-1$ and $|\hat{\mu}(x)-\mu(x)|\leq \Delta_\mu$ for every $x\in \Nn$.
 Thus, for any $\delta>0$, the probability that  is lower-bounded by $$\lambda(u,t)\Pr(\hat{\myset{F}}(u)=v)\geq \frac{\mu(v)-\delta}{1+(n-1)\delta}\Pr(\Delta_\mu<\delta).$$
By taking $\delta>0$ smaller than $\mu(v)$, we have by \eqref{weak} that there exists a $\tau^*$ s.t.~for all $\tau>\tau^*$ the probability that $\oracle(u,v,t)$ takes place at timeslot $\tau$ is bounded away from zero, and the lemma follows.
\end{proof}
\sloppy Thus if $t$ is a target then, after a finite time, for any two $u,v\in \Nn$ the ordered partition $A_1,\ldots,A_j$ returned by $\myset{O}_t.\order()$ will respect the relationship between $u,v$. In particular for $u\in A_i$,$v\in A_{i'}$, if $u\eqvl{t}v$ then $i=i'$, while if $u\lst{t}v$ then $i<i'$. As a result, the estimated rank of an object $u\in A_i$ w.r.t.~$t$ will satisfy
$$\hat{r}_t(u)=\!\!\!\!\sum_{x\in \myset{T}: x\lst{v} u}\!\!\!\!\hat{\mu}(x)+\!\!\!\!\sum_{x\in\Nn\setminus\myset{T}:x\in A_{i'},i'\leq i}\!\!\!\!\hat{\mu}(x) =r_t(u)+O(\Delta_\mu)$$
\emph{i.e.} the estimated rank will be close to the true rank, provided that $\Delta_\mu$ is small. Moreover, as in Lemma~\ref{normalization}, it can be shown that  $$\hat{Z}_v \leq 1\!+\!\log^{-1} \hat{\mu}(v)=1\!+\!\log^{-1} [\mu_v+O(\Delta_\mu)]$$ for $v\in \myset{N}$. From these, for $\Delta_\mu$ small enough,
we have that for $u,v\in \myset{N}$, 
$$\hat{\ell}_u(v) = [\ell_u(v)+O(\Delta_\mu)](1-\epsilon)+ \epsilon \frac{1}{n-1}.$$
Following the same steps as the proof of Theorem~\ref{th:small-world} we can show that, given that $\Delta_\mu\leq \delta$, the expected search cost is upper bounded by
$\frac{6c^3HH_{\max}}{(1-\epsilon) +O(\delta)}$. This gives us that
$$\bar{C}(\tau) \leq \Big[\frac{6c^3HH_{\max}}{(1-\epsilon)} +O(\delta)\Big]\Pr(\Delta_\mu\leq \delta)+ \frac{n-1}{\epsilon} \Pr(\Delta_\mu>\delta) $$
where the second part follows from the fact that, by using the uniform distribution with probability $\epsilon$, we ensure that the cost is stochastically upper-bounded by a geometric r.v.~with parameter $\frac{\epsilon}{n-1}$. Thus, by \eqref{weak}, 
$$\limsup_{\tau\to\infty} \bar{C}(\tau)\leq \frac{6c^3HH_{\max}}{(1-\epsilon)} +O(\delta). $$
As this is true for all small enough delta, the theorem follows.\hspace*{\stretch{1}}

\fussy

\section{Extensions}\label{sec:extensions}
In this section we discuss two possible extensions to the problem of content search through comparisons. The first one is about empowering the 	comparison oracle, namely, assuming that one has access to a stronger oracle which is able to return the most similar object to the target among a set of objects. If we choose the size of the set to be equal to two, we are back to our previous framework.  The second one is about content search when we lift the  assumption that objects are embedded in a metric space.
\subsection{Content Search Beyond Comparison Oracle}
A \emph{proximity oracle} is an oracle that, given a set $A$ of size at most $\kappa$   and a target $t$, returns the closest object to $t$. More formally, 
\begin{equation}\label{proxi-oracle}
\text{\oracle}(A,t)= x, \qquad \text{if } x\preccurlyeq y, \forall x,y\in A.
\end{equation}
Note that the comparison oracle is a special case of the proximity oracle where $|A|=2$. Moreover, accessing $\kappa$ times the comparison oracle, one can implement the proximity oracle. 

\begin{theorem} \label{th:content-search-proxi} Given a demand $\lambda$, consider the memoryless and independent selection policy $$\prob(\myset{F}(\history_k,x_k)=(w_1,w_2,\dots,w_\kappa))=\prod_{i=1}^\kappa\ell_{x_{k}}(w_i)$$ where $\ell_{x_{k}}(w_i)$ is given by \eqref{shortcutdistr}. Then the cost of greedy content search is bounded as follows:
\begin{displaymath}
\bar{C}_{\Ff}\leq \frac{6 c^3(\mu)}{\kappa} \cdot H(\mu)\cdot  H_{\max}(\mu).
\end{displaymath}
\end{theorem}
\begin{proof}
We assume that the target  is object $t$ and the content search starts from $s$. The only way that the greedy search moves from the current object $x_k$ is by proposing a set $A$ that contains an closer to $t$. 
 Like in Section~\ref{th:content-search}, we are in particular interested in bounding the probability  that the rank of the proposed object is roughly half the rank of the  current object. This way we can compute how fast we make progress in our search.  

As the search moves from $s$  to $t$ we say that the search is in phase  $j$ when the rank of the current object $x_k$ is between $2^j\mu(t)$ and $2^{j+1}\mu(t)$. We can write  $C_{\Ff}(s,t)$ as $$C_{\myset{F}}(s,t)=X_1+X_2+\cdots +X_{\log \frac{1}{\mu(t)}},$$ where $X_j$ denotes the number of comparisons done by comparison oracle in phase $j$. Let us consider a particular phase $j$ and denote $I$  the set of objects whose ranks from $t$ are at most $r_t(x_k)/2$. Moreover, let the proposed set by the selection policy be $\myset{F}(\history_k,x_k)=(w_1,w_2,\dots,w_\kappa)$.
Note that phase $j$ will terminate if one of the objects $(w_1,w_2,\dots,w_\kappa)$ is from set $I$. We denote by $F_i, 1\leq i\leq \kappa,$ the event that $w_i\in I$. Since $F_i$'s are independent, the probability that phase $j$ terminates is
\begin{align*}
\sum_{(w_1,w_2,\dots,w_\kappa)\in I^\kappa} \prob(F_1\cup F_2 \cup \dots \cup F_\kappa) 
&\geq \kappa \left(\sum_{w_i\in I}\ell_{x_{k}}(w_i)\right)\\
\end{align*}
 To bound the last expression we can use \eqref{prob-phase}. Hence, with probability at least $\kappa/(6c^3(\mu)H_{\max}(\mu))$,  phase $j$ will terminate. 
 This defines a geometric random variable which yields to the fact that on average the number of queries needed to halve the rank is at most $6c(\mu)^3H_{\max}/\kappa$ or $\E[X_j|s,t]\leq 6c(\mu)^3H_{\max}/\kappa.$ Taking average over the demand $\lambda$, we can conclude that the average number of comparisons is less than $$\bar{C}_{\myset{F}}\leq 6c^3(\mu)H_{\max}(\mu) H(\mu)/\kappa.$$ 
\end{proof}


\subsection{Content Search Beyond Metric Spaces}
 Similarity between objects is a well defined relationship even if the objects are not embedded in a metric space. More specifically, the notation $x\preccurlyeq_z y$ simply states that $x$ is more similar to $z$ than $y$. 

If the only information given about the underlying space is the similarity between objects, then the maximum we can hope for is for each object $x\in \Nn$ sort other objects $\Nn\setminus y$ according to their similarity to $x$. 

Given the demand $\lambda$, the target set $\Tt$ is completely specified. For any $y\in \Tt$  let us define the \textit{rank} as follows:
$$r_x(y)=|\{z:z\in\Tt, z\preccurlyeq_x y\}|.$$
We say that $y\in \Tt$ is the $k$-th closest object to $x$ if $r_x(y)=k$.
First not that the rank is in general asymmetric, \textit{i.e.,} $r_x(y)\neq r_y(x)$. Second, the triangle inequality is not satisfied in general, \textit{i.e.,} $r_x(y)\nleq r_y(z)+ r_z(x)$. However the approximate inequality as introduced in \cite{Lifshits08} is always satisfied. More precisely, we say that the  \textit{disorder factor} $D(\mu)$ is the smallest $D$ such that we have the \textit{approximate triangle inequality} $$r_x(y)\leq D(r_z(y)+ r_z(x)),$$ for all $x,y,z\in \Tt$. The factor $D(\mu)$ basically quantifies the non-homogeneity of the underlying space when the only give information is order of objects.
 Let the selection policy for the non-metric space be defined as follows:
 \begin{equation}\label{selection:non-metric}
\prob(\myset{F}(\history_k,x_k)=w)\propto\frac{1}{r_{x_k}(w)},
\end{equation}
for $w\in \myset{T}$. In case $w\notin \myset{T}$ we define $\prob(\myset{F}(\history_k,x_k)=w)$ to be zero.


It is of high interest to see whether we can still navigate through the database when the characterization of the underlying space is unknown and only the similarity relationship between objects is provided. This is the main theme of the next theorem. 

\begin{theorem}\label{th:non-metric}
Consider the above selection policy. Then for any demand $\lambda$, the cost of greedy content search is bounded as 
\begin{displaymath}
\bar{C}_{\Ff}\leq 7 D(\mu) \log^2|\myset{T}|.
\end{displaymath}
\end{theorem}

The proof of this Theorem is given below. 
Note again that the selection policy is memoryless. Furthermore, it is \textit{universal} in a sense that using this selection policy for any kind of demands guarantees the search that only depends on the cardinality of target set and its disorder factor. For instance, this selection policy is useful when  the target set is only known a priory and the demand is not fully specified.  

\begin{proof}The  selection policy in the non-metric space scenario is given \eqref{selection:non-metric} which implies that only objects in the target set $\Tt$ are going to be proposed by the algorithm. Therefore, except for the starting point $x_0=s$, the algorithms navigates only through the target set.  The probability of proposing $w\in \Tt$  when $x_k$ is the current object of the search is given by
$$\prob(\myset{F}(\history_k,x_k)=w)=\frac{1}{Z_{x_k}}\frac{1}{r_{x_k}(w)},$$
where  $Z_{x_k}=\sum_{w\in \Tt}r^{-1}_{x_k}(w)$. Consequently, 
$$ Z_{x_k} = \left\{ 
\begin{array}{ll}
H_{|\Tt|-1} & \text{if } x_k\in\Tt,\\
H_{|\Tt|} & \text{if } x_k\notin\Tt,\end{array}\right.
$$
where $H_n$ is the $n$-th harmonic number. Hence, $z_{x_k}\leq 2 \log |\Tt|$.
As the search moves from $s$  to $t$ we say that the search is in phase  $j$ when the rank of the current object $v\neq s$ with respect to $t$ is $2^j\leq r_{t}(v)\leq2^{j+1}$. Clearly, there are only $\log|\Tt|$ different phases. The greedy search algorithm keeps proposing to the oracle until it finds another object closer to $t$. We can write $C_{\Ff}(s,t)$ as 
\begin{displaymath}
C_{\Ff}(s,t)=X_1+X_2+\cdots +X_{\log |\Tt|}+X_s,
\end{displaymath}
where $X_s$ denotes the number of comparisons done by oracle  at the starting point until it goes to an object $u\in\Tt$ such that $r_s(u)\leq r_s(t)$. As before $X_j$ ($j>0$) is the number of comparisons done by oracle until it goes to the next phase. 

We need to differentiate between the starting point of the process and the rest of it. Since unlike other objects proposed by the algorithm, the starting object $s$ may not be in the target set.  Let the rank of $t$ with respect to $s$ be $k$,\textit{i.e.,} $r_s(t)=k$. Then, the probability that the greedy search algorithm proposes an object $v\in \Tt$ such that $r_s(v)\leq r_s(t)$ is $\sum_{j=1}^{k}\frac{j}{H_{|\Tt|}}\leq \frac{1}{2\log|\Tt|}.$ As a result $\E[X_s|s,t]\leq 2\log|\Tt|$. This is the average number of comparisons performed by the oracle until the greedy search algorithm escapes from  the starting object $s$. 

Let the current object $v\neq s$ be in phase $j$. We denote by $$ I = \left\{u: u\in \myset{T}, r_t(u)\leq \frac{r_t(v)}{2}\right\},$$ the set of objects whose rank from $t$ is at most $r_t(v)/2$. Clearly, $|I|=r_t(v)/2$. The probability that the greedy search proposes an object $u\in I$ (and hence going to the next phase) is at least $$ \sum_{u\in I}\frac{1}{2\log|\Tt|}\frac{1}{r_v(u)}\overset{(a)}{\geq}\frac{r_t(v)}{4\log|\Tt|D(\mu)(r_t(u)+r_t(v))},$$ where in $(a)$ we used the approximate triangle inequality. Since for $u\in I$, we have $r_t(u)\leq r_t(v)/2$, the probability of going from $v$ to the next phase is at least $6D\log|\Tt|$. Therefore, $\E[X_j|s,t]\leq 6D\log |\Tt|$. 

Using the linearity of expectation,
$$\E[C_{\myset{F}}(s,t)]\leq 6D\log^2|\Tt|+2\log|\Tt|\leq 7D\log^2|\Tt|.$$
The above conditional expectation does not depend on the demand $\lambda$. Hence, the expected search cost for any demand is bounded as  $\E[C_{\myset{F}}]\leq 7D\log^2|\Tt|$. 
\end{proof}

\section{Conclusions}\label{section:conclusions}
In this work, we initiated a study of \textsc{CTSC} and \text{SWND} under heterogeneous demands, tying performance to the topology and the entropy of the target distribution.
Our study leaves several open problems, including  improving upper and lower bounds for both \textsc{CSTC} and \textsc{SWND}. Given the relationship between these two, and the NP-hardness of \textsc{SWND}, characterizing the complexity of \textsc{CSTC} is also interesting. Also, rather than considering restricted versions of \textsc{SWND}, as we did here, devising approximation algorithms for the original problem is another possible direction.

Earlier work on comparison oracles eschewed metric spaces altogether, exploiting what where referred to as \emph{disorder inequalities} 
\cite{Lifshits08,Lifshits09,Suhas}. Applying these under heterogeneity is also a promising research direction. 
Finally, trade-offs between space complexity and the cost of the learning phase vs.~the costs of answering database queries are investigated in the above works, and the same trade-offs could be studied in the context of heterogeneity.

\bibliographystyle{ieeetr}
\bibliography{references}
\end{document}